\documentclass{article}
\PassOptionsToPackage{numbers, compress}{natbib}

     \usepackage[final]{neurips_2023}

\usepackage[verbose=true,letterpaper]{geometry} 
\AtBeginDocument{
  \newgeometry{
    textheight=9in,
    textwidth=5.5in,
    top=1in,
    headheight=12pt,
    headsep=25pt,
    footskip=30pt
  }
}
\usepackage{times}

\usepackage{nicefrac}       
\usepackage{microtype}      

\usepackage[utf8]{inputenc}
\usepackage[T1]{fontenc}
\usepackage{microtype}
\usepackage{graphicx}
\usepackage{subfigure}
\usepackage{booktabs}
\usepackage[colorlinks=true,citecolor=blue,linkcolor=blue]{hyperref}

\usepackage{url}
\usepackage{nicefrac}
\usepackage{listings}
\usepackage{mathtools}
\mathtoolsset{showonlyrefs}
\usepackage{amsfonts}
\usepackage{amssymb}
\usepackage{amsmath}
\usepackage{amsthm}
\usepackage{bookmark}
\usepackage{docmute}
\usepackage[acronym, nomain]{glossaries}
\usepackage{indentfirst}
\usepackage{bm}
\usepackage{here}
\usepackage{booktabs}       
\usepackage{cancel}
\usepackage{multirow}
\usepackage{wrapfig}
\usepackage{lipsum}
\usepackage{thmtools}
\usepackage{thm-restate}
\usepackage {tablefootnote}

\usepackage{natbib}
\usepackage{algorithm,algorithmic}

\newtheorem{assump}{Assumption}

\newtheorem{rem}{Remark}

\declaretheorem[name=Theorem]{thm}
\declaretheorem[name=Lemma]{lem}

\DeclareMathOperator*{\argmin}{\mathop{\rm argmin}}

\DeclarePairedDelimiterX{\KL}[2]{\mathrm{KL}[}{]}{#1\;\delimsize\|\;#2}

\DeclarePairedDelimiterX\braket[2]{\langle}{\rangle}{#1 \delimsize\vert #2}

\usepackage{xcolor}         
\usepackage[para]{footmisc}

\title{Time-Independent Information-Theoretic Generalization Bounds for SGLD}

\author{%
  Futoshi Futami\thanks{Equal contribution.}\\
    Osaka University / RIKEN AIP \\
  \texttt{futami.futoshi.es@osaka-u.ac.jp} \\
  \And
  Masahiro Fujisawa$^{\color{blue}*}$\thanks{Corresponding author.}\\
  RIKEN AIP \\
  \texttt{masahiro.fujisawa@riken.jp}
  }

\begin{document}

\maketitle
\setcounter{footnote}{0}

\begin{abstract}
We provide novel information-theoretic generalization bounds for stochastic gradient Langevin dynamics (SGLD) under the assumptions of smoothness and dissipativity, which are widely used in sampling and non-convex optimization studies.
Our bounds are time-independent and decay to zero as the sample size increases, regardless of the number of iterations and whether the step size is fixed.
Unlike previous studies, we derive the generalization error bounds by focusing on the time evolution of the Kullback--Leibler divergence, which is related to the stability of datasets and is the upper bound of the mutual information between output parameters and an input dataset.
Additionally, we establish the first information-theoretic generalization bound when the training and test loss are the same by showing that a loss function of SGLD is sub-exponential.
This bound is also time-independent and removes the problematic step size dependence in existing work, leading to an improved excess risk bound by combining our analysis with the existing non-convex optimization error bounds.
\end{abstract}

\section{Introduction}
Stochastic optimization, including stochastic gradient descent (SGD), is central to realizing practical large-scale or deep-learning models.
There are currently considerable active discussions on accurately determining the generalization performance of models trained by SGD or its variants.
In particular, stochastic gradient Langevin dynamics (SGLD)~\cite{Gelfand91, welling2011bayesian, Raginsky2017}, a noisy variant of SGD, has garnered much attention in this type of study since it provides a useful theoretical framework for generalization error analysis based on the Langevin diffusion context~\cite{Raginsky2017}.
Our study aims to contribute to a more accurate understanding and evaluation of the generalization performance for SGLD.

There are two main approaches to generalization analysis in SGLD.
One is the \emph{information-theoretic analysis} proposed by \citet{russo16} and \citet{Xu2017}, by which a generalization error bound is derived using the mutual information (MI) between the learned parameters and the training dataset. 
Recently, some extensions using gradient information have been made to investigate the generalization properties of SGLD, for example, upper-bounding the MI with the norm of gradients~\cite{Pensia2018} and the sum of gradient variances~\cite{Negrea2019, Wang2021,wang2022on,wang2023}.
Information-theoretic generalization bounds are applicable to a wide range of noisy iterative algorithms such as differentially private SGD~\cite{Feldman18} and stochastic gradient Hamiltonian Monte Carlo~\cite{cheni14} modified to include a noisy momentum.

The other approach is \emph{stability analysis}, by which the effects of changes in the learning algorithm due to the addition or removal of a single training data point on the generalization performance are investigated.
\citet{Raginsky2017} derived the non-asymptotic generalization and excess risk bound of SGLD via the exponential ergodicity of Langevin diffusion.
Starting with the study by \citet{Raginsky2017}, there have been many attempts to improve the generalization analysis in SGLD from the stability perspective, such as those by \citet{zhang17b, mou18a} and \citet{Li2020On}.

Unfortunately, these existing generalization bounds are \emph{time-dependent}; namely, they diverge with increasing number of  iterations unless the step size is adjusted so that the order of bound values is $\mathcal{O}(n^{-1})$ or $\mathcal{O}(n^{-1/2})$, where $n$ is the sample size (see Section~\ref{sec:related} for details).
\citet{Farghly2021} attempted to avoid this problem by Wasserstein stability analysis through reflection coupling~\cite{Eberle21} under the smoothness and dissipativity~\cite{hale88} assumptions commonly used in sampling and non-convex optimization communities~\cite{Raginsky2017,Zhang21}.
Although their bounds bypass the divergence problem when taking a supremum over time, the geometry introduced for the reflection coupling yields an unnatural dependence on step size, resulting in a vacuous bound as the step size decreases (see Table~\ref{tab_exiting_compare}).

In this paper, we provide novel generalization bounds for SGLD under smooth and dissipative loss functions obtained by the information-theoretic approach.
We focus on the upper bound of the MI, namely, the Kullback--Leibler (KL) divergence between the distributions of parameters learned from different training datasets. We then analyze its time evolution caused by the update of the SGLD algorithm through the Fokker--Planck (FP) equation (Lemma~\ref{lem_MI_time}).
On the basis of this analysis, we obtain time-independent generalization error bounds that decay to zero as $n \rightarrow \infty$ regardless of the number of iterations or whether the step size is fixed (Theorem~\ref{main_thm1_time_independent} and Corollary~\ref{cor:gen_err_nonsurrogate}).
Conventional information-theoretic generalization bounds~\cite{Raginsky2017,Pensia2018,Wang2021,wang2023} are derived by bounding the MI between the parameters at all iterations and the training dataset.
Therefore, these bounds grow linearly with the number of iterations, resulting in a time-dependent generalization error bound. Our analysis based on time evolution eliminates both this linearity issue and the unnatural dependence on step size (the inverse of step size) in the time-independent bound of \citet{Farghly2021}.

Another contribution is providing the first information-theoretic generalization bound and the excess risk bound when the same loss is used for training and the generalization performance evaluation.
In the conventional information-theoretic approach, deriving generalization error bounds under this setting was challenging owing to the unknown tail behavior of a loss function of SGLD. 
We overcome this difficulty with our discovery that a smooth and dissipative loss function of SGLD is sub-exponential.

\section{Preliminaries}
\label{sec:Preliminaries}

\subsection{Problem settings and stochastic gradient Langevin dynamics}
We represent random variables in capital letters, such as $X$, and deterministic values in lowercase letters, such as $x$, and express the Euclidean inner product and distance as $\cdot$ and $\|\cdot\|$. Let $\mu$ be an unknown generating distribution on the instance space $\mathcal{Z}$ and $w \in \mathcal{W}\subseteq \mathbb{R}^{d}$ be the $d$-dimensional parameters such as weights of neural networks, where $\mathcal{W}$ is the space of the parameters.
We consider a loss function $l:\mathcal{W} \times \mathcal{Z}\to \mathbb{R}$ and the following optimization problem:
\begin{align*}
\min_{w \in \mathcal{W}} L_\mu(w) \coloneqq \mathbb{E}_{Z} [l(w,Z)] = \int_{\mathcal{Z}}l(w,Z)\mathrm{d}\mu(z),
\end{align*}
which cannot be computed since $\mu$ is unknown.
Instead, we typically minimize the empirical risk estimated using the dataset $S \coloneqq \{Z_i\}_{i=1}^{n}$:
\begin{align*}
\min_{w\in\mathcal{W}} L_S(w)\coloneqq \frac{1}{n}\sum_{i=1}^n l(w,Z_i),
\end{align*}
where $\{Z_i\}_{i=1}^{n}$ are independent and identically distributed (i.i.d.) samples from $\mu$, i.e., $Z_i\overset{\mathrm{i.i.d.}}{\sim} \mu$.

\paragraph{Stochastic gradient Langevin dynamics.}
In this paper, we use the SGLD algorithm~\cite{welling2011bayesian} to solve the empirical risk minimization. SGLD utilizes the gradient information of the loss function; however, some loss functions, such as the $0$-$1$ loss, are not differentiable.
In this case, it is common to use the differentiable surrogate loss function $f:\mathcal{W} \times \mathcal{Z}\to \mathbb{R}$ (e.g., the cross-entropy loss) and minimize the following empirical risk: 
$F_S(w)\coloneqq\frac{1}{n}\sum_{i=1}^{n}f(w,Z_i)$.
Given a mini-batch $B \subset [n] \coloneqq \{1,\cdots,n\}$ with $k=|B| \leq n$, we define its mini-batch version as
\begin{align*}
F(w,B)\coloneqq\frac{1}{k}\sum_{i\in B}f(w,Z_i).
\end{align*}
The SGLD algorithm updates the parameters using the following recursion:
\begin{align*}
W_{t+1}=W_t-\eta_t \nabla F(W_{t},B_{t})+\sqrt{2\beta_t^{-1}\eta_t}\xi_t, \quad W_0\sim P_{W_0},
\end{align*}
where $P_{W_0}$ is a given initial distribution, $\nabla F(w,B)$ is a stochastic gradient, $t$ is the number of iterations, $\eta_t$ is the step size, $\beta_t$ is the inverse temperature, and $(B_t)_{t=0}^\infty$ is an i.i.d.~sequence of random variables distributed uniformly on $\{B\subset [n]: |B|=k\}$.
In addition, $(\xi_t)_{t=0}^\infty$ is an i.i.d.~sequence of standard Gaussian random variables, i.e., $\xi_t\sim \mathcal{N}(0,\mathbf{I}_d)$, where $\mathbf{I}_d$ is the $d$-dimensional identity matrix.
The output parameters $W \in\mathcal{W}$ obtained using SGLD can be seen as the samples from a conditional distribution $P_{W|S}: \mathcal{Z}^n \to \mathcal{W}$. We express the $t$-th output of the SGLD algorithm as $W_t$.

\subsection{Expected generalization error and its bounds}
The focus of this paper is the {\em expected generalization error}, defined as
\begin{align}
\label{eq:gen_error}
\mathrm{gen}(\mu,P_{W|S};L)\coloneqq\mathbb{E}_{S,W}[L_\mu(W)-L_S(W)],
\end{align}
where the expectation is taken over the joint distribution of $(S,W)$, i.e., $\mu^{n}\otimes P_{W|S}$.

\paragraph{Information-theoretic generalization bounds.}
\citet{russo16} and \citet{Xu2017} have shown that Eq.~\eqref{eq:gen_error} can be bounded by the MI between the input dataset $S$ and the output parameters $W$ under the following sub-Gaussian assumption.
\begin{assump}[sub-Gaussian losses]
    \label{assump:sub_gaussian}
    A loss function $l(w,Z)$ is sub-Gaussian under $Z \sim \mu$ for all $w\in\mathcal{W}$, that is, there is a positive constant $\sigma^2_g$ such that $\log \mathbb{E}_Z[\exp (\lambda(l(w,Z)-\mathbb{E}l(w,Z)))]\leq \lambda^2 \sigma_g^2/2$ for all constant $\lambda \in \mathbb{R}$.
\end{assump}
For example, bounded or Lipschitz-continuous loss functions satisfy this assumption. 
Assumptions regarding the tail behavior of the loss function distributions as in the above are necessary for the information-theoretic generalization error analysis.
\citet{bu2020tightening} have investigated information-theoretic generalization bounds with another tail-behavior assumption such as sub-exponential losses.

We introduce the following standard information-theoretic generalization bound.
\begin{thm}[\citet{russo16} and \citet{Xu2017}]
Suppose that Assumption~\ref{assump:sub_gaussian} holds. 
Then, we have
\begin{align}
\label{eq_MI}
|\mathrm{gen}(\mu,P_{W|S};L)| \leq \sqrt{  \frac{2\sigma_g^2}{n}I(W;S)},
\end{align}
under a training dataset $S = \{Z_{i}\}_{i=1}^{n}$ and the algorithm's output $W$, where $I(W;S)$ is the MI between $W$ and $S$.
\end{thm}

In the SGLD context, $I(W,S)$ of Eq.~\eqref{eq_MI} can be upper-bounded in a form that incorporates the gradient variance~\cite{Pensia2018,Negrea2019,Wang2021,wang2023}.
Given the output of the $T$-th iterate of the SGLD algorithm, $W_T$, the following upper bound can be obtained.
\begin{thm}[Modified bound of \citet{Pensia2018}]\label{thm_existing_info}
Let $f(\cdot,z)$ be an $L$-Lipschitz continuous function, namely, there is a constant $L > 0$ such that $\|f(w,z)-f(\bar{w},z)\|\leq L\|w-\bar{w}\|$ holds for all $w,\bar{w} \in \mathcal{W}$ and all $z\in \mathcal{Z}$.
Then, we obtain
\begin{align}
\label{info_bound_exist}
I(W_{T};S) \overset{\mathrm{(i)}}{\leq} \sum_{t=0}^T \frac{d}{2}\log \left(1+\frac{\beta_t\eta_t}{d}\mathrm{Var}[\nabla f(W_{t},B_t)|W_t]\right)
\leq \sum_{t=0}^T \frac{d}{2}\log \left(1+\frac{\beta_t\eta_t L^2}{d}\right),
\end{align}
where $\mathrm{Var}[\nabla f(W,B)|W]\coloneqq\mathbb{E}\mathbb{E}_{B}[\|\nabla_W f(W,B)-\mathbb{E}_{B}[\nabla_W f(W,B)]\|^2|W]$ is the conditional variance.
\end{thm}
Note that bound (i) can be obtained by the data-processing inequality \cite{cover2012element}.

The merit of such information-theoretic generalization bounds is that we can evaluate the bound value using the empirically estimated gradient variance per iteration.
However, unfortunately, from Eq.~\eqref{info_bound_exist}, this bound is \emph{time-dependent}; namely, the bound value can diverge unless the gradient variance or $\beta_t\eta_t$ approaches $0$ as $T \to \infty$. This is due to the data-processing inequality when deriving upper bound (i) in Eq.~\eqref{info_bound_exist}. 
By the data-processing inequality, we obtain $I(W_T;S)\leq I(W^{(T)};S)$, where $W^{(T)}\coloneqq(W_0,W_1,\cdots,W_T)$ denotes the joint random variables appearing in all the iterations in the algorithm. Since $W^{(T)}$ is treated simultaneously, the bound is inevitably linear in $T$.

Another limitation of the information-theoretic approach appears in the setting where training losses ($f$) are also used for performance evaluation, which is often employed in sampling and non-convex optimization studies of SGLD~\cite{Raginsky2017,Xu2018}.
In this setting, the generalization error is defined as
\begin{align}
\label{eq:gen_error_F}
    \mathrm{gen}(\mu,P_{W|S};F)\coloneqq\mathbb{E}_{S,W}[F_\mu(W)-F_S(W)],
\end{align}
where $F_{\mu} \coloneqq \mathbb{E}_{Z}[f(w,Z)]$.
We cannot conduct the information-theoretic analysis for Eq.~\eqref{eq:gen_error_F} because the tail behavior of the distribution of the training loss is unclear.

\paragraph{Time-independent generalization bounds for Eq.~\eqref{eq:gen_error_F}.}
To solve the above problems, \citet{Farghly2021} provided the generalization error bounds of Eq.~\eqref{eq:gen_error_F} from the stability perspective under the following assumptions widely used in the non-convex optimization analysis of SGLD~\cite{Raginsky2017,Xu2018,kinoshita2022}.
\begin{assump}[Smoothness]
\label{asm_smooth}
For each $z\in \mathcal{Z}$, $f(\cdot,\!z)$ is differentiable and $M$-smooth. That is, there is a positive constant $M$ for all $w,\ \bar{w} \in\mathcal{W}$ and all $z\in \mathcal{Z}$ such that
\begin{align*}
\|\nabla f(w,z)-\nabla f(\bar{w},z)\|\leq M\|w-\bar{w}\|.
\end{align*}
\end{assump}
\begin{assump}[Dissipativity~\cite{hale88}]
\label{asm_dissipative}
For each $z\in \mathcal{Z}$, $f(\cdot,z)$ is $(m,b)$-dissipative.~\footnote{This assumption holds not only for (strongly) convex losses but also for many practically used non-convex loss functions~\cite{Mou22}. For example, it applies to non-convex loss functions with $l_{2}$ constraints and likelihood functions that satisfy Poincar\'{e} inequality~\cite{bakry2013analysis, vempala2019rapid}.} That is, there are positive constants $m$ and $b$ for all $w\in\mathcal{W}$ and $z\in \mathcal{Z}$ such that
\begin{align*}
m\|w\|^2-b\leq \nabla f(w,z)\cdot w.
\end{align*}
\end{assump}
The discussion regarding loss functions that satisfy Assumption~\ref{asm_dissipative} is presented in Appendix~\ref{app:add_info_dissipative}.

Hereafter, we eliminate the time dependence of the step size and temperature by setting $\eta_{t} = \eta$ and $\beta_t=\beta$. 
With this notation, \citet{Farghly2021} derived the following generalization bound.
\begin{thm}[\citet{Farghly2021}]\label{thm:farghly}
Suppose that Assumptions~\ref{asm_smooth} and \ref{asm_dissipative} hold.
Assume that the initial law of $W_0$ has a finite fourth moment $\sigma$.
Then, if $\eta \leq 1/2m$, for any $T\in\mathbb{N}$, we have
\begin{align}
\label{eq:farghly}
|\mathrm{gen}(\mu,P_{W_T|S};F)| < C_1 \left(\eta T\wedge \frac{n(C_2+1)}{n-k}\right)\left(\frac{k}{n\eta^{1/2}}+\eta^{1/2}\right),
\end{align}
where $(x \wedge y) = \min\{x,y\}$, and $C_1$ and $C_2$ are the positive constant terms w.r.t.~$\{M,m,b,d,\beta,\sigma\}$ and $\{M,m,b,d,\beta\}$, respectively.
\end{thm}
\citet{Farghly2021} utilized the Wasserstein stability on the basis of the contraction property of Langevin diffusion under reflection coupling.
The important technique to derive the above bound is that we only focus on $W_T$ differently from $W^{(T)}$ of the information-theoretic approach when deriving the contraction property.
In this way, the resulting bounds do not suffer from divergence as $T \to  \infty$; however, it still has a problem.
That is, Eq.~\eqref{eq:farghly} depends on the factor $\eta^{-1/2}$, which implies that it becomes vacuous or even diverges with decreasing $\eta (=\eta_{T})$ as $T \to \infty$.

In this paper, we propose new generalization bounds to address the drawbacks of the information-theoretic and stability-based approaches.
Specifically, the proposed bounds are established on the basis of the two expected generalization errors outlined in Eqs.~\eqref{eq:gen_error} and \eqref{eq:gen_error_F}, which remain time-independent and do not diverge as the step size decreases.

\section{Time-independent generalization error bound for SGLD}
\label{sec_sgld}
Here, we explain our time-independent bound of $\mathrm{gen}(\mu, P_{W|S};L)$ for SGLD.
We first introduce the main result (Section~\ref{subsec:main_res}) and then summarize its proof outline (Sections~\ref{subsec:methodology} and \ref{subsec:proof_outline_lemma}). 
Finally, in Section~\ref{sec_examples}, we provide a detailed discussion on our bound with concrete examples.

\subsection{Main result}
\label{subsec:main_res}
Our key idea is to derive the generalized error bound using the FP equation.
To use the FP equation, we impose the following regularity condition for $P_{W_{0}}$.
\begin{assump}[Regularity of the initial distribution]
\label{regularity_FP}
The initial distribution of $W_0$: $P_{W_{0}}$ is a Gaussian distribution~\footnote{The Gaussian assumption can be relaxed, e.g., to a Gaussian mixture, in the theorems and corollaries shown in this paper. The detailed discussions are provided in Appendix~\ref{app:relax_gauss}.} with a finite variance $s^2>0$, which is independent of $\eta$ and $T$.
\end{assump}
Our analysis is also grounded in the time evolution of the FP equation using the logarithmic Sobolev inequality (LSI)~\cite{bakry2013analysis} associated with $\pi$ described as follows.
We state that $\pi$ satisfies the LSI with constant $c_{\mathrm{LS}}$, if for any $\rho \ll \pi$, the following relation holds:
\begin{align}
\mathrm{KL}(\rho|\pi)\leq c_{\mathrm{LS}} \mathbb{E}\|\nabla \log\rho-\nabla \log \pi \|^2\notag.
\end{align}
\citet{Raginsky2017} showed the existence of $c_{\mathrm{LS}}$ under Assumptions~\ref{asm_smooth}, \ref{asm_dissipative}, and $\beta\geq 2/m$. 
Note that $c_{\mathrm{LS}}$ is expressed by the problem-dependent constant (see Appendix~\ref{subsec:FP_equations} for details).

We now introduce our generalization error bound.
\begin{restatable}{thm}{maingenbound}
\label{main_thm1_time_independent}
Suppose that Assumptions~\ref{assump:sub_gaussian}, \ref{asm_smooth}, \ref{asm_dissipative}, and \ref{regularity_FP} are satisfied.
Then, for any $\beta\geq 2/m$ and $\eta \in (0,1 \wedge \frac{m}{5M^2}\wedge 4\beta c_{\mathrm{LS}})$ and any $T\in\mathbb{N}$, we have
\begin{align}
\label{eq:our_bound_surrogate}
&|\mathrm{gen}(\mu,P_{W_T|S};L)|\leq \sqrt{\frac{2c_1 \sigma_g^2}{n}\left(1\wedge \frac{\eta T}{4\beta c_{\mathrm{LS}}}\right)\left(V_\nabla + c_2\right)}, 
\end{align}
where $c_1$, $c_2$, and $\displaystyle V_\nabla$ are the positive constant terms w.r.t.~$\{M,m,b,d,\beta,s^{2}\}$.
\end{restatable}
The above theorem shows $|\mathrm{gen}(\mu,P_{W_T|S};L)|=\mathcal{O}(\sqrt{(\eta T\wedge 1)/n})$, which implies time independence since it does not diverge even if $T \to \infty$ and thus converges as $n \to \infty$.

In Eq.~\eqref{eq:our_bound_surrogate}, the term $V_\nabla$ corresponds to \emph{stability}, which is expressed as the upper bound of the difference of the expected conditional gradients with respect to changes in training datasets at each iteration.
This shows a certain similarity to existing information-theoretic generalization bounds, such as Theorem~\ref{thm_existing_info}, expressed by the variance of gradients with respect to the training datasets.
This similarity is discussed in detail in Section~\ref{sec_examples}.
Additionally, detailed information on the explicit expression of $c_1$, $c_2$, and $V_\nabla$ can be found in Appendix~\ref{app_sec_sgld}.

\subsection{Proof outline of Theorem~\ref{main_thm1_time_independent}}
\label{subsec:methodology}
In this section, we present how to derive our bound in Theorem~\ref{main_thm1_time_independent}.
Our aim here is to share the ideas behind our analysis and an outline of the proof, providing the detailed proof in Appendix~\ref{app_sec_sgld}.

We adopt the information-theoretic approach and focus on the MI in Eq.~\eqref{eq_MI}. By using the Jensen inequality, we have the following upper bound of the MI:
\begin{align}
\label{eq:mi_kl_bound}
 I(W_{T};S)\leq \mathbb{E}_{S, S'}\mathrm{KL}(P_{W_T|S}|P_{W_T|S'}),
\end{align}
where $S$ and $S'$ are random variables drawn independently from $\mu^n$, and $\mathrm{KL}(P_{W_T|S}|P_{W_T|S'})$ is the KL divergence from $P_{W_T|S'}$ to $P_{W_T|S}$. 
Note that this KL divergence indicates the stability of the learned parameter from two datasets, $S$ and $S'$. 
We also note that $P_{W_T|S'}$ can be regarded as the data-dependent prior. 
Thus, this KL divergence is tighter than that of the data-independent prior, which is often used in the probably approximately correct (PAC)-Bayes bound~\footnote{We can confirm this from the fact that $\displaystyle \mathbb{E}_{S, S'}\mathrm{KL}(P_{W_T|S}|P_{W_T|S'})=\mathbb{E}_{S}\mathrm{KL}(P_{W_T|S}|P_{W_T})-\mathbb{E}_{S'}\mathrm{KL}(P_{W_T|S'}|P_{W_T})$, where $P_{W_T}$ is a data-independent prior distribution.}.

The key idea is to analyze \emph{the time evolution of the KL divergence}, which is summarized in the following lemma:
\begin{lem}
\label{lem_MI_time}
Suppose that the same assumptions in Theorem~\ref{main_thm1_time_independent} hold. 
Then, for any $t\in \mathbb{N}$, we have
\begin{align}
\label{mi_evolution}
\mathrm{KL}(P_{W_t|S}|P_{W_t|S'})\leq e^{-\frac{\eta}{4\beta c_{\mathrm{LS}}}}\mathrm{KL}(P_{W_{t-1}|S}|P_{W_{t-1}|S'})+\eta V_{\Delta} + c_3 \eta, 
\end{align} 
where $V_{\Delta}$ and $c_3$ is the constant term w.r.t.~$\{M,m,b,d,\beta, s^{2}\}$.
\end{lem}
We will discuss the details of Lemma~\ref{lem_MI_time} in Section~\ref{subsec:proof_outline_lemma}.
By recursively applying Eq.~\eqref{mi_evolution} from $t=0$ to $T$, we obtain
\begin{align}
\label{eq_mutual_info_final_main}
\mathrm{KL}(P_{W_T|S}|P_{W_T|S'})&\leq \frac{1-e^{-\frac{\eta T}{4\beta c_{\mathrm{LS}}}}}{1-e^{-\frac{\eta}{4\beta c_{\mathrm{LS}}}}}\eta\left(V_\nabla+c_3 \right)\overset{\mathrm{(i)}}{\leq} 4\beta c_{\mathrm{LS}}\left(1\wedge \frac{\eta T}{4\beta c_{\mathrm{LS}}}\right)\frac{1}{1-\frac{\eta}{4\beta c_{\mathrm{LS}}}}\left(V_\nabla+c_3\right),
\end{align}
which is based on the fact that $\mathrm{KL}(P_{W_0|S}|P_{W_0|S'})=0$.
Note that bound (i) can be obtained from %
$e^{-\frac{\eta}{4\beta c_{\mathrm{LS}}}}<1-\frac{\eta}{4\beta c_{\mathrm{LS}}}+\frac{\eta^{2}}{16\beta^2 c_{\mathrm{LS}}^2}$ when $\frac{\eta}{4\beta c_{\mathrm{LS}}}\leq 1$, and $e^{-\frac{\eta T}{4\beta c_{\mathrm{LS}}}}\geq 1-\frac{\eta T}{4\beta c_{\mathrm{LS}}}$. 

\subsection{Proof outline of Lemma~\ref{lem_MI_time} under the continuous Langevin diffusion}
\label{subsec:proof_outline_lemma}
Here, we organize our ideas for the proof of Lemma~\ref{lem_MI_time} that are important in the derivation of Theorem~\ref{main_thm1_time_independent}.
For simplicity, we now provide an intuitive explanation and an outline of the proof under the continuous Langevin diffusion setting.
Note that the results of Theorem~\ref{main_thm1_time_independent} and Lemma~\ref{lem_MI_time} are based on the SGLD setting, and their proofs are shown in Appendix~\ref{app_sec_sgld}.

The Langevin diffusion is defined as
\begin{align}\label{LD_st}
\mathrm{d}W_{t}=-\nabla F(W_{t},S)\mathrm{d}t+\sqrt{2\beta^{-1}}\mathrm{d}H_t,
\end{align}
where $\mathrm{d}H_t$ is the standard Brownian motion in $\mathbb{R}^d$.
Note that, in this context, $t$ expresses the continuous time and the \emph{full-batch} gradient $\nabla F(W_{t}, S)$ is used.
The stationary distribution of Eq.~\eqref{LD_st} is given as the Gibbs distribution $\pi(\mathrm{d}w) \propto \exp(-\beta F(w,S))$.

With some abuse of notation, let us denote $P_{W_t|S}$ as the conditional distribution obtained using Eq.~\eqref{LD_st} and express its density as $\rho_t$. Then, the FP equation \cite{bakry2013analysis} for Eq.~\eqref{LD_st} can be obtained as
\begin{align}
\label{eq:FP_rho}
\frac{\partial \rho_t(w,t)}{\partial t}=\nabla \cdot \bigg(\frac{1}{\beta}\nabla\rho_t + \rho_t \nabla F(w,S)\bigg).
\end{align}
Similarly, we can define the Langevin diffusion when we use dataset $S'$ and the conditional distribution using that diffusion as $P_{W_t|S'}$ with the density $\gamma_{t}$, obtaining the FP equation in the form of $\rho_t$ replaced by $\gamma_{t}$ in Eq.~\eqref{eq:FP_rho}.

Now we analyze the time evolution of $\mathrm{KL}(P_{W_t|S}|P_{W_t|S'})=\mathrm{KL}(\rho_t|\gamma_t)$ at time $t$, i.e., $\partial \mathrm{KL}(\rho_t|\gamma_t)/\partial t$.
By utilizing the FP equations of $\rho_t$ and $\gamma_t$ and the Cauchy--Schwartz inequality, we obtain the following upper bound:
\begin{align}
\label{eq_kl_opt}
\frac{\partial \mathrm{KL}(\rho_t|\gamma_t)}{\partial t}\leq -\frac{1}{2\beta}\mathbb{E}\|\nabla \log\rho_t-\nabla\log \gamma_t\|^2 +\frac{\beta}{2} \mathbb{E}\| \nabla F(W_t,S)- \nabla F(W_t,S')\|^{2}.
\end{align}
The second term on the right-hand side of Eq.~\eqref{eq_kl_opt} represents the stability of the gradient with respect to the randomness of the training dataset $S, S'\sim \mu^n$, which leads to $V_\nabla$ in Lemma~\ref{lem_MI_time} under the SGLD setting.
Hereafter, we define $\mathbb{E}\| \nabla F(W_t,S)- \nabla F(W_t,S')\|^{2}$ as $\widetilde{V}_{\nabla_t}$.

By introducing $\nabla \log \pi(w)$ into $\mathbb{E}\|\nabla \log\rho_t-\nabla\log \gamma_t\|^2$ in Eq.~\eqref{eq_kl_opt}, we obtain
\begin{align}
\label{eq:kl_opt_bound}
\frac{\partial \mathrm{KL}(\rho_t|\gamma_t)}{\partial t}
&\leq -\frac{1}{4\beta}\mathbb{E}\|\nabla \log\rho_t-\nabla \log \pi \|^2 +\frac{1}{2\beta}\mathbb{E}\|\nabla \log \pi\|^2+\frac{1}{\beta}\mathbb{E}\nabla \log\rho_t \nabla \log\gamma_t +\frac{\beta}{2} \widetilde{V}_{\nabla_t} \notag \\
&\leq -\frac{1}{4\beta c_{\mathrm{LS}}} \mathrm{KL}(\rho_t|\pi) +\frac{1}{2\beta}\Omega(\rho_t,\gamma_t,\pi) +\frac{\beta}{2} \widetilde{V}_{\nabla_t} \notag\\
&\leq -\frac{1}{4\beta c_{\mathrm{LS}}} \left(\mathrm{KL}(\rho_t|\gamma_t)+\mathbb{E}\log\frac{\gamma_t}{\pi}\right) +\frac{1}{2\beta}\Omega(\rho_t,\gamma_t,\pi) +\frac{\beta}{2} \widetilde{V}_{\nabla_t},
\end{align}
where the first inequality is from the fact that $-x^2\leq -\|x-y\|^2/2+y^2$ for $x, y\in\mathbb{R}^d$ and the second one is from the LSI. We introduced $\Omega(\rho_t,\gamma_t,\pi)\coloneqq \mathbb{E}_{\rho_t}\|\nabla \log \pi\|^2+2\mathbb{E}_{\rho_t}\nabla \log\rho_t \cdot \nabla \log\gamma_t$ to simplify the notation.

By integrating $e^{\frac{t}{4\beta c_{\mathrm{LS}}}} \frac{\partial \mathrm{KL}(\rho_{t}|\gamma_{t})}{\partial t}$ in Eq.~\eqref{eq:kl_opt_bound} over $t \in [0,\eta]$ and rearranging it, we obtain
\begin{align}
\mathrm{KL}(\rho_\eta|\gamma_\eta)&\leq e^{\frac{-\eta}{4\beta c_{\mathrm{LS}}}}\mathrm{KL}(\rho_0|\gamma_0) \notag\\
&\quad + \! \int_0^{\eta}e^{\frac{-(\eta-t)}{4\beta c_{\mathrm{LS}}}}\!\left(\frac{\beta}{2} \widetilde{V}_{\nabla_t}\!-\frac{1}{4\beta c_{\mathrm{LS}}} \mathbb{E}\log\frac{\gamma_t}{\pi}\!+\frac{1}{2\beta}\Omega(\rho_t,\gamma_t,\pi)\!\right)\mathrm{d}t. 
\label{Eq_residue}
\end{align}
In Appendix~\ref{app_sec_sgld}, we show that the terms related to $\pi$ in Eq.~\eqref{Eq_residue} can be bounded by using the techniques of \citet{Raginsky2017} and \citet{vempala2019rapid}.

We next derive an upper bound for the following terms
in Eq.~\eqref{Eq_residue}: $\mathbb{E}[\log \frac{\gamma_t}{\pi}] $ and $\Omega(\rho_t,\gamma_t,\pi)$ by using the \emph{parametrix method} for the FP equation~\cite{friedman2008partial, pavliotis2014stochastic}, which allows us to expand the FP equation's solution via the heat kernel.
On the basis of this expansion, we can upper bound Eq.~\eqref{Eq_residue} as
\begin{align}\label{eq_parametric}
\int_0^{\eta}e^{\frac{-(\eta-t)}{4\beta c_{\mathrm{LS}}}}\left(-\frac{1}{4\beta c_{\mathrm{LS}}} \mathbb{E}\log\frac{\gamma_t}{\pi}+\frac{1}{2\beta}\Omega(\rho_t,\gamma_t,\pi)\right)\mathrm{d}t   \leq  \mathcal{O}(\eta).
\end{align}
By combining Eq.~\eqref{eq_parametric} with Eq.~\eqref{Eq_residue}, we obtain the continuous version of Lemma~\ref{lem_MI_time}.

The same procedure can be used for the SGLD setup.
The difference from the continuous Langevin diffusion case is that the discretization errors and the effects of using a stochastic gradient are taken into account, resulting in the appearance of an additional constant \footnote{This constant is evaluable (see \citet{vempala2019rapid} or \citet{kinoshita2022}).} in the above bounds (see Appendix~\ref{app_sec_sgld} for details).

\subsection{Additional discussion on our bound in terms of stability}
\label{sec_examples}
We conclude this section by presenting further discussion on our bound in terms of stability with a concrete example.

As shown in Eq.~\eqref{eq:mi_kl_bound}, the information-theoretic generalization bound is closely related to the stability in KL divergence under the different training datasets.
However, our bound in Theorem~\ref{main_thm1_time_independent} incorporates the constant term $c_{2}$, which is irrelevant to stability, alongside the stability term $V_{\nabla}$.
If we can avoid the occurrence of $c_2$, the resulting upper bound of $\mathrm{KL}(P_{W_t|S}|P_{W_t|S'})$ would be dominantly expressed by $V_{\nabla}$, and as a result, we may obtain a bound where the relationship between generalization and stability is more directly represented.

The problematic constant term $c_2$ arises from $c_3\eta$ in Lemma~\ref{lem_MI_time} analyzing the time evolution of stability in KL divergence. Specifically, the term $c_3\eta$ is the byproduct of treating the general dissipative function using LSI.
Actually, it is possible to avoid the problematic constant term $c_3\eta$ and derive bounds that are evaluated solely on the basis of stability-related metrics in specific examples, such as strongly convex or bounded (non-convex) losses with $l_2$-regularization.
For simplicity, we show this fact using the following theorem under the Langevin diffusion (LD) setting, where the probability induced by Eq.~\eqref{LD_st} is expressed as $P_{W_T|S}$.
\begin{restatable}{thm}{maincontinuous}
\label{main_thm2_continuous}
Suppose that Assumptions~\ref{assump:sub_gaussian} and \ref{asm_smooth} are satisfied and that $F(w,z)$ is $R$-strongly convex ($0 < R< \infty$).
Then, for any $T\in\mathbb{R}_{+}$, we have
\begin{align}\label{eq_stability_variance}
\frac{\partial \mathrm{KL}(\rho_t|\gamma_t)}{\partial t}\leq 
-\frac{R}{4}\ \mathrm{KL}(\rho_t|\gamma_t) +\frac{\beta}{2} \mathbb{E}\| \nabla F(W_t,S)- \nabla F(W_t,S')\|^{2},
\end{align}
and
\begin{align}\label{eq_stability_generalizaton_con}
|\mathrm{gen}(\mu,P_{W_T|S};L)|\leq \sqrt{\frac{2\beta \sigma_g^2}{n}\int_{0}^{T}e^{-\frac{(T-t)R}{4}} \mathbb{E}\| \nabla F(W_t,S)- \nabla F(W_t,S')\|^{2}\mathrm{d}t}.
\end{align}
\end{restatable}
A similar bound in Eq.~\eqref{eq_stability_generalizaton_con} (with $R$ replaced by $\lambda/e^{8\beta C}$) can be obtained for bounded non-convex losses with $l_2$-regularization, where $F(w,z) = F_0(w,z) + \frac{\lambda}{2}\|w\|^2$ ($0 < \lambda < \infty$) and $F_0(w,z)$ is $C$-bounded ($0 \leq C < \infty$).
The full proof is summarized in Appendix~\ref{app_convex_bounded}.

When comparing with Lemma~\ref{lem_MI_time}, we can see that, in Eq.~\eqref{eq_stability_variance}, stability-unrelated constants do not appear in the time evolution of KL divergence at each time step. 
Therefore, the resulting generalization bound is also independent of such constants.
Furthermore, when compared with Theorem~\ref{thm_existing_info}, which adds up the stability terms at all time steps, our bound is dominated by the stability terms near the final time step, as those at earlier time steps decrease geometrically by $e^{-\frac{R}{4}}$. 
This indicates that the stability around the initial time steps is of lesser importance in evaluating the final generalization performance.

Note that our bounds are closely related to the bound indicated in Proposition~9 of \citet{mou18a}, which was also derived by focusing on stability.
The bound of \citet{mou18a} primarily assesses generalization errors focusing on the \emph{gradient norm} near the conclusion of training. In contrast, our bounds evaluate it through the norm of \emph{differences in gradients}, emphasizing the state in the proximity of training completion.
In other words, our bound allows for the evaluation of generalization errors using a stability measure that is more closely related to generalization performance than the gradient norm.
This benefit originates from our approach, which tracks the time evolution of MI-related stability in Eq.~\eqref{eq_stability_variance} on the basis of information-theoretic generalization bounds, in contrast to the PAC-Bayes bounds derived from the direct analysis of stability measures as in \citet{mou18a}.

\section{Generalization analysis for SGLD directly using a training loss}
In this section, we consider the setting that the generalization performance is measured by a training loss $f$ directly as in Eq.~\eqref{eq:gen_error_F}.
We show that this is possible by demonstrating that loss functions of SGLD are sub-exponential under smooth and dissipative assumptions (Section~\ref{subsec:sub_exp_loss}).
On the basis of this fact, we obtain for the first time an information-theoretic generalization bound of SGLD that is similar to Theorem~\ref{main_thm1_time_independent}.
Finally, combining these results with existing optimization error bounds provides an excess risk bound with improved convergence (Section~\ref{subsec:excess_risk}).

\subsection{Smooth and dissipative loss function of SGLD is sub-exponential}
\label{subsec:sub_exp_loss}
To perform an information-theoretic analysis for SGLD, it is necessary to know the tail behavior of $f(W,Z)$.
Our contribution here is showing that a loss function of SGLD under smooth and dissipative assumptions is sub-exponential.
\begin{restatable}{thm}{subexp}
\label{thm:sub_exp}
Suppose that Assumptions~\ref{asm_smooth}, \ref{asm_dissipative} and \ref{regularity_FP} are satisfied.
Let $P_{W_T}=\mathbb{E}_S[P_{W_T|S}]$ be the marginal distribution of the output obtained using the SGLD algorithm at the $T$-th iteration.
Then, for any $T\in\mathbb{N}$, $f(W_T,Z)$ is sub-exponential under the distribution $P_{W_T} \otimes \mu$.
That is, there exist positive constants $\sigma^2_e$ and $\nu$ w.r.t.~$\{m,\beta,M,b,d,s^2\}$~\footnote{The explicit form of $\sigma^2_e$ and $\nu$ can be seen in Appendix~\ref{App_sec_sub_expo_main}.} such that
\begin{align}
\log \mathbb{E}_{W_T\otimes Z}\left[e^{\lambda (f(W_T,Z)-\mathbb{E}_{W_T\otimes Z}[f(W_T,Z)])}\right]\leq  \frac{\sigma^2_e\lambda^2}{2}\quad for\ all\ |\lambda|<\frac{1}{\nu}.
\end{align}
\end{restatable}
\begin{proof}[Proof sketch]The complete proof is shown in Appendix~\ref{App_sec_sub_expo_full}.
First, note that under Assumptions~\ref{asm_smooth} and \ref{asm_dissipative}, for any $z\in \mathcal{Z}$, we obtain
\begin{align}
\label{eq:f_bound_main}
\frac{m}{3}\|w\|^2-\frac{b}{2}\log 3 \leq f(w,z)\leq \frac{M}{2}\|w\|^2+M\sqrt{\frac{b}{m}}\|w\|+A,
\end{align}
where $A$ is a positive constant (see Lemma~\ref{lem_function_bound} in Appendix~\ref{App_sec_sub_expo} for its explicit form).
We can also show that for any $p\in\mathbb{N}$, we have
\begin{align}\label{eq:p_th_moment}
(\mathbb{E}\|W_T\|_2^{p})^{1/p}\leq C\left(\mathbb{E}\|W_0\|_2^{p}\right) ^{\frac{1}{p}}+C\sqrt{\frac{p+\beta b+d}{\beta m}},
\end{align}
where $C$ is a universal constant. This implies that $W_T$ is a sub-Gausssian random variable~\cite{vershynin2018high}.
To show the sub-exponential property, we directly upper-bound  $\mathbb{E}_{W_T\otimes Z}[e^{\lambda(f(W_T,Z)-\mathbb{E}_{W_T\otimes Z}[f(W_T,Z)])}]$ by considering the Taylor expansion of the exponential moment and using Eqs.~\eqref{eq:f_bound_main} and ~\eqref{eq:p_th_moment}.
\end{proof}
\begin{rem}
In previous information-theoretic analysis studies~\cite{Pensia2018,Negrea2019,Wang2021}, it is often assumed that a loss function $l(w,Z)$ is sub-Gaussian under the distribution $\mu$ for all $w\in\mathcal{W}$.
In contrast, Theorem~\ref{thm:sub_exp} holds under the distribution $P_{W_T} \otimes \mu$, not conditioned on $w\in\mathcal{W}$.
\end{rem}

We can interpret the sub-exponential property of SGLD intuitively as follows. Under Assumptions~\ref{asm_smooth} and \ref{asm_dissipative}, the loss function grows at most as a quadratic function shown in Eq.~\eqref{eq:f_bound_main}.
The conditional distribution of the parameters follows the Gaussian distribution, and the square of the Gaussian random variable is known as the chi-square ($\chi^2$) random variable~\cite{wainwright_2019}. 
According to these facts, we expect that the behavior of the loss function resembles that of the $\chi^2$-random variable; therefore, it is sub-exponential since the $\chi^2$-distribution is also sub-exponential~\cite{wainwright_2019}.
Theorem~\ref{thm:sub_exp} validates this intuition.

\subsection{Generalization bounds for SGLD using the same loss for training and evaluation}
\label{subsec:excess_risk}
On the basis of Theorem~\ref{thm:sub_exp}, we can derive the following information-theoretic generalization bound for SGLD even if a surrogate loss is not used.
In contrast to Theorem~\ref{main_thm1_time_independent}, an assumption regarding the tail behavior of a loss function such as Assumption~\ref{assump:sub_gaussian} is not necessary.
\begin{restatable}{cor}{nonsurrogate}
\label{cor:gen_err_nonsurrogate}
Suppose that Assumptions~\ref{asm_smooth}, \ref{asm_dissipative}, and \ref{regularity_FP} are satisfied.
Then, for any $\beta\geq 2/m$, $\eta \in  (0,1 \wedge \frac{m}{5M^2}\wedge 4\beta c_{\mathrm{LS}})$, and $T\in\mathbb{N}$, we obtain
\begin{align}
|\mathrm{gen}(\mu,P_{W_T|S};F)| \leq \Psi^{*-1}\left(\frac{c_1}{n}\left(1\wedge \frac{\eta T}{4\beta c_{\mathrm{LS}}}\right)\left(V_\nabla + c_2\right)\right),
\end{align}
where 
\begin{align*}
    \Psi^{*-1}(y)=\begin{cases}
    \sqrt{2\sigma_e^2 y}\quad \mathrm{if}\ y\leq \frac{\sigma_e^2}{2\nu}\\
    \nu y+\frac{\sigma_e^2}{2\nu}\quad \mathrm{otherwise}
    \end{cases},
\end{align*}
$c_1$ and $c_2$ are the same as in Theorem~\ref{main_thm1_time_independent}, and $\sigma_e^2$ and $\nu$ are the same as in Theorem~\ref{thm:sub_exp}.
\end{restatable}
\begin{proof}[Proof sketch]
This is the direct consequence of the sub-exponential property from Theorem~\ref{thm:sub_exp} and the upper bound of MI in Eq.~\eqref{eq:mi_kl_bound} (see Appendix~\ref{app_proof_gen_without_surrogate} for the complete proof).
\end{proof}
\begin{rem}
Despite the assumptions of Corollary~\ref{cor:gen_err_nonsurrogate} being the same as those made by \citet{Farghly2021} except for the initial distribution and step size, the resulting bound becomes $0$ as $n\rightarrow \infty$ without being dependent on inverse stepsize.
\end{rem}
 
We conclude this section by introducing our excess risk bound.
Let us define the excess risk as follows: $\mathrm{Excess}(\mu,P_{W|S})\coloneqq\mathbb{E}_{W,S}[F_\mu(W)-F_\mu(w^*)]$, where $w^*=\argmin_{w\in\mathcal{W}}F_\mu(w)$.
Under this definition, we derive the following upper bound for the excess risk by utilizing Corollary~\ref{cor:gen_err_nonsurrogate}.
\begin{restatable}{cor}{excess}
\label{cor:excess_risk}
Suppose that Assumptions~\ref{asm_smooth}, \ref{asm_dissipative}, and \ref{regularity_FP} are satisfied.
Then, for any $\beta\geq 2/m$, $\eta \in  (0,1 \wedge \frac{m}{5M^2}\wedge 4\beta c_{\mathrm{LS}})$, and $T\in\mathbb{N}$, we obtain
\begin{align}
\mathrm{Excess}(\mu,P_{W_T|S}) = \mathcal{O}\bigg(\sqrt{\frac{(\eta T\wedge 1)}{n}}+e^{-\eta T/c_{LS}}+\sqrt{\eta}+c_{\mathrm{err}} \bigg),
\end{align}
where $c_{\mathrm{err}}$ is the positive constant w.r.t.~$\{M,m,b,d,\beta\}$ corresponding to the optimization error.
\end{restatable}
We show the complete proof in Appendix~\ref{subsec:proof_excess}.
In contrast with the existing excess risk studies, our bound does not diverge with increasing $t$ owing to the time-independent generalization bound in Corollary~\ref{cor:gen_err_nonsurrogate}.

\section{Related studies and discussion}
\label{sec:related}
In this section, we compare our generalization bounds with those in related studies.
Table~\ref{tab_exiting_compare} shows the order of each bound value along with its assumptions for a loss function.
\paragraph{SGLD analysis with/without changing losses.}
The existing generalization error bounds in Table~\ref{tab_exiting_compare} are time-dependent; namely, we need to impose restrictive conditions for the step size $\eta$ in terms of $t$ to achieve a generalization bound that decays to zero with increasing sample size~\citep{Raginsky2017, Pensia2018, Negrea2019, Wang2021} (see the right column in Table~\ref{tab_exiting_compare}).
Some important applications of SGLD do not satisfy these conditions.
For instance, the short-run Markov chain Monte Carlo~\cite{nijkamp2019learning} method used in energy-based models~\cite{Hinton02} adopts SGLD with a \emph{fixed} step size. Another example is the cyclic SGLD~\cite{zhangcyclical} used in deep learning, where the step size is \emph{periodically increased or decreased} to facilitate escape from local optima.

\citet{Farghly2021} first analyzed the generalization error of SGLD by using smoothness and dissipative assumptions, which are broadly used in sampling and non-convex optimization studies~\cite{Raginsky2017,Xu2018, Chau21, Zhang21}.
Their bound is time-independent; the bound does not diverge with time and achieves the order $\mathcal{O}(n^{-1/2})$.
However, the bound depends on the inverse of step size $\eta^{-1/2}$ owing to the reflection coupling~\cite{Eberle21}, which results in the unnatural behavior of decreasing $\eta$ with increasing $t$.
\citet{Farghly2021} also derived a bound that does not suffer from this problem by assuming the Lipschitz loss function with weight decay; however, these assumptions excessively restrict the class of loss functions and algorithms. In contrast to these bounds, our bound is time-independent and does not require scaling $\eta$, $t$, and $n$ to achieve $\mathcal{O}(n^{-1/2})$.

\begin{table}[t]
\centering
\caption{Comparison of our bounds with those in existing studies.
Our bounds are time-independent and bounded even if $\eta \rightarrow 0$.
(\textbf{I}) denotes the information-theoretic approach and (\textbf{S}) denotes the stability analysis approach. The symbol * means that the sub-Gaussian assumption is unnecessary for our bounds when using the same loss for training and generalization performance evaluation. Namely, our bounds can be derived under more relaxed assumptions for a loss function in this case.}
\label{tab_exiting_compare}
\scalebox{.74}{
\begin{tabular}{lll}
\hline
\multicolumn{1}{c}{Study} & \multicolumn{1}{c}{Assumptions for a loss function}          & \multicolumn{1}{c}{Expected generalization error bound} \\ \hline \hline
(\textbf{S})~\citet{Raginsky2017} (Thm.~2.1.)      & Dissipative, Smoothness                 & $\mathcal{O}(\eta t+e^{-\eta t/c}+1/n)$                                                        \\
(\textbf{S})~\citet{mou18a} (Thm.~1.)                         & Bounded, Lipschitz                      & 
$\mathcal{O}(\sqrt{\eta t}/n)$ \\
(\textbf{S})~\citet{mou18a} (Thm.~2.)                         & Lipschitz, Sub-Gaussian, (Weight decay)~\tablefootnote{The order of the bound in \citet{mou18a} varies with the choice of regularization parameters and decay factors. In this paper, we adopt the order of this bound in Table~1 of \citet{Farghly2021}. For a more comprehensive discussion, we refer to Section~5.2 of \citet{mou18a}.}   &  $\mathcal{O}(\sqrt{\eta \log (t+1) / n})$                                                       \\
(\textbf{I})~\citet{Pensia2018} (Cor.~1.)                         & Lipschitz, Sub-Gaussian                 &  $\mathcal{O}(\sqrt{\eta t / n})$                                                       \\
(\textbf{I})~\citet{Negrea2019} (Thm.~3.1.)                         & Sub-Gaussian                 & 
$\mathcal{O}(\sqrt{\eta t / n})$ \\
(\textbf{S})~\citet{Farghly2021} (Thm.~3.1.)                          & Lipschitz, Smoothness, Weight decay     & $\mathcal{O}((\eta t\wedge 1)(1/n+\sqrt{\eta}))$                                                        \\
(\textbf{S})~\citet{Farghly2021} (Thm.~4.1.)                         & Dissipative, Smoothness                 & $\mathcal{O}((\eta t\wedge 1)(\sqrt{\eta^{-1}}/n+\sqrt{\eta}))$                                                        \\
(\textbf{I})~\citet{Wang2021} (Thm.~1.)                         & Sub-Gaussian                 &
$\mathcal{O}(\sqrt{\eta t / n})$ \\
(\textbf{I})~Ours (Thm.~\ref{main_thm1_time_independent} and Cor.~\ref{cor:gen_err_nonsurrogate})                         & Dissipative, Smoothness, Sub-Gaussian* & 
$\mathcal{O}(\sqrt{(\eta t\wedge 1)/n)}$ \\ \hline
\end{tabular}
}
\end{table}

\paragraph{Time evolution analysis of MI via FP equation.}
The analysis of SGLD using the FP equation has been successfully used in the convergence analysis of SGLD~\cite{vempala2019rapid, kinoshita2022}. 
These studies present analyses of the discretization errors and convergence properties of the unadjusted Langevin algorithm, SGLD, and variance-reduction SGLD (SVRG-LD)~\cite{Dubey2016}, comparing them with the continuous Langevin dynamics through the FP equation.

In generalization error analysis, the FP equation is mainly used to analyze the time evolution of KL divergence appearing in a generalization bound on the basis of the stability approach.
\citet{Li2020On} analyzed the time evolution of the KL divergence between the probability densities of the parameters obtained from two training datasets that differ by only one data point under the bounded loss assumption.
\citet{mou18a} also studied the KL divergence and Hellinger divergence, and they derived a generalization error bound on the basis of the PAC-Bayes notion~\cite{McAllester03}.
Our idea is similar to these: we analyze the time evolution of the KL divergence between the probability densities of the parameters obtained from two training datasets. 
The differences between our approach and other approaches are twofold.
First, we do not assume weight decay or Lipschitz continuity but instead derive our analysis assuming smoothness and dissipativity.
Second, the KL divergence we analyzed is tighter than that of the PAC-Bayes bound with data-independent prior dealt by \citet{mou18a}.

\section{Limitations and future work}
\label{sec:conclusion}
In this paper, we provide a generalization analysis of SGLD, where Gaussian noise is a fundamental assumption for our theoretical results.
Thus, it is difficult to extend our analysis to other noisy iterative algorithm variants with a different noise, such as differentially private SGD with Laplace or uniform noise~\cite{wang2023}.
Another limitation of this study is that we have estimated the sub-exponential parameter roughly with respect to the dimensions of the model parameters.
Further investigation of the sub-exponentiality of smooth and dissipative losses, and improvement of the dependence on dimensionality, are crucial for enhancing the practicality of our generalization bounds.
The sub-exponential property of a loss function is expected to be helpful in fields other than generalized error analysis. 
For example, this property opens up room for new theoretical analysis policies that employ useful concentration and transport inequalities~\cite{wainwright_2019} in the sampling and optimization context.
We hope that the analysis presented in this paper goes beyond generalization analysis and provides valuable insights into understanding the characteristics of machine learning. 
\begin{ack}
We sincerely appreciate the anonymous reviewers for their insightful feedback.
FF was supported by JSPS KAKENHI Grant Number JP23K16948.
FF was supported by JST, PRESTO Grant Number JPMJPR22C8, Japan.
MF was supported by RIKEN Special Postdoctoral Researcher Program.
MF was supported by JST, ACT-X Grant Number JPMJAX210K, Japan.
\end{ack}
\bibliography{main}

\begin{thebibliography}{44}
\providecommand{\natexlab}[1]{#1}
\providecommand{\url}[1]{\texttt{#1}}
\expandafter\ifx\csname urlstyle\endcsname\relax
  \providecommand{\doi}[1]{doi: #1}\else
  \providecommand{\doi}{doi: \begingroup \urlstyle{rm}\Url}\fi

\bibitem[Amit et~al.(2022)Amit, Epstein, Moran, and Meir]{amit2022}
R.~Amit, B.~Epstein, S.~Moran, and R.~Meir.
\newblock Integral probability metrics {PAC}-bayes bounds.
\newblock In \emph{Advances in Neural Information Processing Systems}, 2022.

\bibitem[Bakry et~al.(2013)Bakry, Gentil, and Ledoux]{bakry2013analysis}
D.~Bakry, I.~Gentil, and M.~Ledoux.
\newblock \emph{Analysis and Geometry of Markov Diffusion Operators}, volume
  348.
\newblock Springer Science \& Business Media, 2013.

\bibitem[Bu et~al.(2020)Bu, Zou, and Veeravalli]{bu2020tightening}
Y.~Bu, S.~Zou, and V.~V. Veeravalli.
\newblock Tightening mutual information-based bounds on generalization error.
\newblock \emph{IEEE Journal on Selected Areas in Information Theory},
  1\penalty0 (1):\penalty0 121--130, 2020.

\bibitem[Chau et~al.(2021)Chau, Moulines, R\'{a}sonyi, Sabanis, and
  Zhang]{Chau21}
N.~H. Chau, \'{E}. Moulines, M.~R\'{a}sonyi, S.~Sabanis, and Y.~Zhang.
\newblock On stochastic gradient {L}angevin dynamics with dependent data
  streams: {T}he fully nonconvex case.
\newblock \emph{SIAM Journal on Mathematics of Data Science}, 3\penalty0
  (3):\penalty0 959--986, 2021.

\bibitem[Chen et~al.(2014)Chen, Fox, and Guestrin]{cheni14}
T.~Chen, E.~Fox, and C.~Guestrin.
\newblock Stochastic gradient {H}amiltonian {M}onte {C}arlo.
\newblock In \emph{Proceedings of the 31st International Conference on Machine
  Learning}, volume~32, pages 1683--1691, 2014.

\bibitem[Cover and Thomas(2012)]{cover2012element}
T.~M. Cover and J.~A. Thomas.
\newblock \emph{Elements of Information Theory}.
\newblock John Wiley \& Sons, 2012.

\bibitem[Deck and Kruse(2002)]{deck2002parabolic}
T.~Deck and S.~Kruse.
\newblock Parabolic differential equations with unbounded coefficients -- {A}
  generalization of the parametrix method.
\newblock \emph{Acta Applicandae Mathematica}, 74:\penalty0 71--91, 2002.

\bibitem[Dubey et~al.(2016)Dubey, Reddi, Williamson, Poczos, Smola, and
  Xing]{Dubey2016}
K.~A. Dubey, S.~J. Reddi, S.~A. Williamson, B.~Poczos, A.~J. Smola, and E.~P.
  Xing.
\newblock Variance reduction in stochastic gradient {L}angevin dynamics.
\newblock In \emph{Advances in Neural Information Processing Systems},
  volume~29, pages 1154--1162, 2016.

\bibitem[Eberle(2021)]{Eberle21}
A.~Eberle.
\newblock Reflection couplings and contraction rates for diffusions.
\newblock \emph{Probability Theory and Related Fields}, 166:\penalty0 851--886,
  2021.

\bibitem[Farghly and Rebeschini(2021)]{Farghly2021}
T.~Farghly and P.~Rebeschini.
\newblock Time-independent generalization bounds for {SGLD} in non-convex
  settings.
\newblock In \emph{Advances in Neural Information Processing Systems},
  volume~34, pages 19836--19846, 2021.

\bibitem[Feldman et~al.(2018)Feldman, Mironov, Talwar, and Thakurta]{Feldman18}
V.~Feldman, I.~Mironov, K.~Talwar, and A.~Thakurta.
\newblock Privacy amplification by iteration.
\newblock In \emph{2018 IEEE 59th Annual Symposium on Foundations of Computer
  Science (FOCS)}, pages 521--532, 2018.

\bibitem[Friedman(2008)]{friedman2008partial}
A.~Friedman.
\newblock \emph{Partial Differential Equations of Parabolic Type}.
\newblock Courier Dover Publications, 2008.

\bibitem[Gelfand and Mitter(1991)]{Gelfand91}
S.~B. Gelfand and S.~K. Mitter.
\newblock Recursive stochastic algorithms for global optimization in
  {$\mathbb{R}^d$}.
\newblock \emph{SIAM Journal on Control and Optimization}, 29\penalty0
  (5):\penalty0 999--1018, 1991.

\bibitem[Hale(1988)]{hale88}
J.~K. Hale.
\newblock \emph{Asymptotic Behavior of Dissipative Systems}.
\newblock Mathematical surveys and monographs. American Mathematical Society,
  1988.

\bibitem[Harutyunyan et~al.(2021)Harutyunyan, Raginsky, Steeg, and
  Galstyan]{harutyunyan2021}
H.~Harutyunyan, M.~Raginsky, G.~Ver Steeg, and A.~Galstyan.
\newblock Information-theoretic generalization bounds for black-box learning
  algorithms.
\newblock In \emph{Advances in Neural Information Processing Systems}, pages
  24670--24682, 2021.

\bibitem[Haussmann and Pardoux(1986)]{Haussmann1986}
U.~G. Haussmann and E.~Pardoux.
\newblock Time reversal of diffusions.
\newblock \emph{The Annals of Probability}, 14\penalty0 (4):\penalty0
  1188--1205, 1986.

\bibitem[Hinton(2002)]{Hinton02}
G.~E. Hinton.
\newblock Training products of experts by minimizing contrastive divergence.
\newblock \emph{Neural Computation}, 14\penalty0 (8):\penalty0 1771--1800,
  2002.

\bibitem[Jiang et~al.(2020)Jiang, Neyshabur, Mobahi, Krishnan, and
  Bengio]{Jiang20}
Y.~Jiang, B.~Neyshabur, H.~Mobahi, D.~Krishnan, and S.~Bengio.
\newblock Fantastic generalization measures and where to find them.
\newblock In \emph{The Tenth International Conference on Learning
  Representations}, 2020.

\bibitem[Kinoshita and Suzuki(2022)]{kinoshita2022}
Y.~Kinoshita and T.~Suzuki.
\newblock Improved convergence rate of stochastic gradient {L}angevin dynamics
  with variance reduction and its application to optimization.
\newblock In \emph{Advances in Neural Information Processing Systems}, pages
  19022--19034, 2022.

\bibitem[Li et~al.(2020)Li, Luo, and Qiao]{Li2020On}
J.~Li, X.~Luo, and M.~Qiao.
\newblock On generalization error bounds of noisy gradient methods for
  non-convex learning.
\newblock In \emph{The Eighth International Conference on Learning
  Representations}, 2020.

\bibitem[Livni(2022)]{livni22}
R.~Livni.
\newblock Information theoretic lower bounds for information theoretic upper
  bounds.
\newblock \emph{arXiv preprint arXiv:2302.04925}, 2022.

\bibitem[McAllester(2003)]{McAllester03}
D.~A. McAllester.
\newblock {PAC}-{B}ayesian stochastic model selection.
\newblock \emph{Machine Learning}, 51\penalty0 (1):\penalty0 5--21, 2003.

\bibitem[Mou et~al.(2018)Mou, Wang, Zhai, and Zheng]{mou18a}
W.~Mou, L.~Wang, X.~Zhai, and K.~Zheng.
\newblock Generalization bounds of {SGLD} for non-convex learning: {T}wo
  theoretical viewpoints.
\newblock In \emph{Proceedings of the 31st Conference on Learning Theory},
  volume~75, pages 605--638, 2018.

\bibitem[Mou et~al.(2022)Mou, Flammarion, Wainwright, and Bartlett]{Mou22}
W.~Mou, N.~Flammarion, M.~J. Wainwright, and P.~L. Bartlett.
\newblock {Improved bounds for discretization of Langevin diffusions:
  Near-optimal rates without convexity}.
\newblock \emph{Bernoulli}, 28\penalty0 (3):\penalty0 1577 -- 1601, 2022.

\bibitem[Negrea et~al.(2019)Negrea, Haghifam, Dziugaite, Khisti, and
  Roy]{Negrea2019}
J.~Negrea, M.~Haghifam, G.~K. Dziugaite, A.~Khisti, and D.~M. Roy.
\newblock Information-theoretic generalization bounds for {SGLD} via
  data-dependent estimates.
\newblock In \emph{Advances in Neural Information Processing Systems},
  volume~32, pages 11015--11025, 2019.

\bibitem[Nijkamp et~al.(2019)Nijkamp, Hill, Zhu, and Wu]{nijkamp2019learning}
E.~Nijkamp, M.~Hill, S.-C. Zhu, and Y.~N. Wu.
\newblock Learning non-convergent non-persistent short-run {MCMC} toward
  energy-based model.
\newblock In \emph{Advances in Neural Information Processing Systems},
  volume~32, pages 5232--5242, 2019.

\bibitem[Pavliotis(2014)]{pavliotis2014stochastic}
G.~A. Pavliotis.
\newblock \emph{Stochastic processes and applications: {D}iffusion processes,
  the Fokker-Planck and Langevin equations}, volume~60.
\newblock Springer, 2014.

\bibitem[Pensia et~al.(2018)Pensia, Jog, and Loh]{Pensia2018}
A.~Pensia, V.~Jog, and P.-L. Loh.
\newblock Generalization error bounds for noisy, iterative algorithms.
\newblock In \emph{2018 IEEE International Symposium on Information Theory
  (ISIT)}, pages 546--550, 2018.

\bibitem[Raginsky et~al.(2017)Raginsky, Rakhlin, and Telgarsky]{Raginsky2017}
M.~Raginsky, A.~Rakhlin, and M.~Telgarsky.
\newblock Non-convex learning via stochastic gradient {L}angevin dynamics: {A}
  nonasymptotic analysis.
\newblock In \emph{Proceedings of the 30th Conference on Learning Theory},
  volume~65, pages 1674--1703, 2017.

\bibitem[Russo and Zou(2016)]{russo16}
D.~Russo and J.~Zou.
\newblock Controlling bias in adaptive data analysis using information theory.
\newblock In \emph{Proceedings of the 19th International Conference on
  Artificial Intelligence and Statistics}, volume~51, pages 1232--1240, 2016.

\bibitem[Shalev-Shwartz et~al.(2009)Shalev-Shwartz, Shamir, Srebro, and
  Sridharan]{Shwartz09}
S.~Shalev-Shwartz, O.~Shamir, N.~Srebro, and K.~Sridharan.
\newblock Stochastic convex optimization.
\newblock In \emph{Proceedings of the 22nd Conference on Learning Theory},
  2009.

\bibitem[Vempala and Wibisono(2019)]{vempala2019rapid}
S.~Vempala and A.~Wibisono.
\newblock Rapid convergence of the unadjusted {L}angevin algorithm:
  {I}soperimetry suffices.
\newblock In \emph{Advances in Neural Information Processing Systems},
  volume~32, pages 8094--8106, 2019.

\bibitem[Vershynin(2018)]{vershynin2018high}
R.~Vershynin.
\newblock \emph{High-dimensional probability: {A}n introduction with
  applications in data science}, volume~47.
\newblock Cambridge university press, 2018.

\bibitem[Wainwright(2019)]{wainwright_2019}
M.~J. Wainwright.
\newblock \emph{High-Dimensional Statistics: A Non-Asymptotic Viewpoint}.
\newblock Cambridge Series in Statistical and Probabilistic Mathematics.
  Cambridge University Press, 2019.

\bibitem[Wang et~al.(2021)Wang, Huang, Gao, and Calmon]{Wang2021}
H.~Wang, Y.~Huang, R.~Gao, and F.~Calmon.
\newblock Analyzing the generalization capability of {SGLD} using properties of
  {G}aussian channels.
\newblock In \emph{Advances in Neural Information Processing Systems},
  volume~34, pages 24222--24234, 2021.

\bibitem[Wang et~al.(2023)Wang, Gao, and Calmon]{wang2023}
H.~Wang, R.~Gao, and F.~P. Calmon.
\newblock Generalization bounds for noisy iterative algorithms using properties
  of additive noise channels.
\newblock \emph{Journal of Machine Learning Research}, 24\penalty0
  (26):\penalty0 1--43, 2023.

\bibitem[Wang and Mao(2022)]{wang2022on}
Z.~Wang and Y.~Mao.
\newblock On the generalization of models trained with {SGD}:
  {I}nformation-theoretic bounds and implications.
\newblock In \emph{The Tenth International Conference on Learning
  Representations}, 2022.

\bibitem[Wang and Mao(2023)]{wang23}
Z.~Wang and Y.~Mao.
\newblock Tighter information-theoretic generalization bounds from
  supersamples.
\newblock In \emph{Proceedings of the 40th International Conference on Machine
  Learning}, volume 202, pages 36111--36137, 2023.

\bibitem[Welling and Teh(2011)]{welling2011bayesian}
M.~Welling and Y.~W. Teh.
\newblock Bayesian learning via stochastic gradient {L}angevin dynamics.
\newblock In \emph{Proceedings of the 28th International Conference on
  International Conference on Machine Learning}, pages 681--688, 2011.

\bibitem[Xu and Raginsky(2017)]{Xu2017}
A.~Xu and M.~Raginsky.
\newblock Information-theoretic analysis of generalization capability of
  learning algorithms.
\newblock In \emph{Advances in Neural Information Processing Systems},
  volume~30, pages 2524--2533, 2017.

\bibitem[Xu et~al.(2018)Xu, Chen, Zou, and Gu]{Xu2018}
P.~Xu, J.~Chen, D.~Zou, and Q.~Gu.
\newblock Global convergence of {L}angevin dynamics based algorithms for
  nonconvex optimization.
\newblock In \emph{Advances in Neural Information Processing Systems},
  volume~31, pages 3126--3137, 2018.

\bibitem[Zhang et~al.(2020)Zhang, Li, Zhang, Chen, and Wilson]{zhangcyclical}
R.~Zhang, C.~Li, J.~Zhang, C.~Chen, and A.~G. Wilson.
\newblock Cyclical stochastic gradient {MCMC} for {B}ayesian deep learning.
\newblock In \emph{The Eighth International Conference on Learning
  Representations}, 2020.

\bibitem[Zhang et~al.(2017)Zhang, Liang, and Charikar]{zhang17b}
Y.~Zhang, P.~Liang, and M.~Charikar.
\newblock A hitting time analysis of stochastic gradient {L}angevin dynamics.
\newblock In \emph{Proceedings of the 30th Conference on Learning Theory},
  volume~65, pages 1980--2022, 2017.

\bibitem[Zhang et~al.(2021)Zhang, Akyildiz, Damoulas, and Sabanis]{Zhang21}
Y.~Zhang, \"{O}.~D. Akyildiz, T.~Damoulas, and S.~Sabanis.
\newblock Nonasymptotic estimates for stochastic gradient {L}angevin dynamics
  under local conditions in nonconvex optimization.
\newblock \emph{Applied Mathematics \& Optimization}, 87\penalty0 (25), 2021.

\end{thebibliography}
\bibliographystyle{plainnat}
\clearpage
\appendix

\section{Notation used in the main paper}
We summarize the notation we used in the main part of this paper.

\begin{table}[th]
\centering
\scalebox{.53}{
\begin{tabular}{cll}
\hline
\multicolumn{1}{c}{\textbf{Category}}               & \multicolumn{1}{c}{\textbf{Symbol}} & \multicolumn{1}{c}{\textbf{Meaning}} \\ \hline \hline
\multirow{21}{*}{Scalars and constants}             &  $n \in \mathbb{N}$ & The sample size \\
                                                    &  $w \in \mathbb{R}$ & Model parameters (deterministic) \\
                                                    &  $w^* \in \mathbb{R}$ & $\argmin_{w\in\mathcal{W}} L_{\mu}(w)$ (deterministic) \\
                                                    &  $W\in \mathbb{R}$ & Model parameters (random variables) \\
                                                    &  $t, T \in \mathbb{N}$ & An iteration of the SGLD algorithm \\
                                                    &  $W^{(T)}$ & The joint random variables appearing in all the iterations until $T$ \\
                                                    &  $d\in \mathbb{R}$ & The number of parameter dimensions \\
                                                    &  $k\in \mathbb{N}$ & The number of samples in a mini-batch $B$ ($\leq n$) \\
                                                    &  $\mathbf{I}_{d}$ & Identity matrix with $d$ rows and $d$ columns \\
                                                    &  $\eta_{t} (= \eta)\in \mathbb{R}$ & The learning rate \\
                                                    &  $\beta_{t} (= \beta)\in \mathbb{R}$ & The inverse temperature \\
                                                    &  $\xi_{t}\in \mathbb{R}$ & Gaussian noise sampled from $\mathcal{N}(0,\mathbf{I}_{d})$ \\
                                                    &  $\sigma_{g}^{2}$ & A positive constant for sub-Gaussian random variables \\
                                                    &  $L$ & A positive constant in the Lipschitz continuous function \\ 
                                                    &  $M$ & A positive constant in the smoothness condition \\
                                                    &  $m, b$ & Positive constants in the dissipative condition \\
                                                    &  $s^{2}$ & A positive finite Gaussian variance for the initial parameter distribution
                                                    $P_{W_{0}}$ \\ 
                                                    & $\sigma \in \mathbb{R}$ & A finite fourth moment of the initial parameters $W_{0}$ \\
                                                    &  $c_{\mathrm{LS}}\in \mathbb{R}$ & The logarithmic Sobolev constant \\
                                                    &  $\widetilde{V}_{\nabla_{t}}\in \mathbb{R}$ & The expected value of $\mathbb{E}\|\nabla F(W_{t},S) - \nabla F(W_{t},S') \|^{2}$ \\
                                                    &  $\sigma^{2}_{e}$ & A positive constant for sub-exponential random variables \\
                                                    &  $c_{\mathrm{err}}\in \mathbb{R}$ & The constant w.r.t.~$\{M,m,b,d,\beta\}$ corresponding to the optimization error \\ \hline
\multirow{9}{*}{Sets and sequences}                 &  $\mathcal{Z}$ & The instance space \\
                                                    &  $\mathcal{W}$ & The parameter space \\
                                                    &  $\mathbb{R}, \mathbb{R}_{+}$ & The set of real numbers and that of positive real numbers \\
                                                    &  $\mathbb{N}$ & The set of natural numbers \\
                                                    &  $[n] \coloneqq \{1,\ldots,n \}$ & The set of all integers between $1$ and $n$ \\
                                                    &  $S, S' \coloneqq \{ Z_{i}\}_{i=1}^{n} (\in \mathbb{R})$ & The i.i.d.~samples from $\mu^{n}$\\
                                                    &  $B \subset [n]$ & A mini-batch set \\
                                                    &  $(B_{t})_{t=1}^{\infty}$ & An i.i.d.~sequence of random variables specifying indexes \\
                                                    &  $(\xi)_{t=1}^{\infty}$ & An i.i.d.~sequence of Gaussian noise $\xi_{t}$ \\ \hline
\multirow{13}{*}{Probability and information theory} &  $\mu$ & An unknown data generating distribution \\
                                                    &  $P_{W|S}$ & A conditional distribution w.r.t.~$W$ given $S$ via SGLD (or the continuous Langevin diffusion) \\
                                                    &  $\mathcal{N}(\mathbf{m},\mathbf{\Sigma})$ & Gaussian distribution with mean $\mathbf{m} \in \mathbb{R}^{d}$ and covariance $\mathbf{\Sigma}\in \mathbb{R}^{d\times d}$ \\
                                                    &  $P \otimes Q$ & The product distribution \\
                                                    &  $I(W;S)$ & The mutual information between $W$ and $S$ \\
                                                    &  $\mathbb{E}_{x}$ & The expectation w.r.t.~$x$ \\
                                                    &  $\mathbb{E}$ & The expectation w.r.t. all randomness \\
                                                    &  $\mathrm{Var}[\nabla f(W,B)|W]$ & The gradient variance w.r.t.~$B$ conditioned by $W$ \\
                                                    &  $\mathrm{Var}[\nabla f(W,B)]$ & The gradient variance w.r.t.~$B$ \\
                                                    &  $\mathrm{KL}(P|Q)$ & The Kullback--Leibler divergence of $P$ from $Q$ \\
                                                    &  $\rho_{t}$ & The density of $P_{W_{t}|S}$ \\
                                                    &  $\gamma_{t}$ & The density of $P_{W_{t}|S'}$ \\
                                                    &  $\pi$ & The Gibbs distribution (stationary distribution of the continuous Langevin diffusion) \\ \hline
\multirow{8}{*}{Functions}                          &  $l: \mathcal{W} \times \mathcal{Z} \to \mathbb{R}$ & an original loss function \\
                                                    &  $f: \mathcal{W} \times \mathcal{Z} \to \mathbb{R}$ & a surrogate loss function \\
                                                    &  $L_{\mu}$, $F_{\mu}$ & The population risk based on an original or a surrogate loss \\
                                                    &  $L_{S}$, $F_{S}$ & The empirical risk based on an original or a surrogate loss \\
                                                    &  $F(w,B)$ & The empirical risk with $l$ or $f$ on a mini-batch $B$ \\
                                                    &  $\nabla F(w,B)$ & The gradient of $F(w,B)$ w.r.t.~$w$ \\
                                                    &  $\mathrm{gen}(\mu, P_{W|S};L)$, $\mathrm{gen}(\mu, P_{W|S};F)$ & The expected generalization error based on an original or a surrogate loss\\
                                                    &  $\mathrm{Excess}(\mu, P_{W|S})$ & The excess risk defined as $\mathbb{E}_{W,S}[F_{\mu}(W) - F_{\mu}(w^*)]$
\end{tabular}
}
\end{table}

\section{Additional information for dissipative losses}
\label{app:add_info_dissipative}
Here, we provide the additional information for losses with dissipativity in Assumption~\ref{asm_dissipative}.

The dissipative assumption plays an essential role in guaranteeing the geometrical convergence of SGLD to the stationary distribution. 
We note that convergence to the stationary distribution is crucial for reducing training error since the stationary distribution in this context corresponds to the Gibbs posterior distribution of the given loss function. 
The dissipative assumption is widely used in the research on sampling or non-convex potential function optimization; thus, it is a fundamental property that enables optimization with SGLD rather than strong constraint conditions for generalization. As \citet{Mou22} discussed, the dissipative assumption is weaker than convexity and strong convexity.

Many non-convex losses commonly used in practice satisfy the dissipative property.
First, all strongly convex and convex losses obviously satisfy dissipativity.
The dissipative losses also include losses that are strongly convex or convex when sufficiently far from zero, that is, there exists $m, R >0$ such that $\|x - y\|^{2} \geq R,$ $(x-y)\cdot(\nabla f(x,z) - \nabla f(y,z)) \geq m \|x-y\|^{2}$ or that $\|x - y\|^{2} \geq R,$ $(x-y)\cdot(\nabla f(x,z) - \nabla f(y,z)) \geq 0$ for all $x$ and $y$ (refer to \citet{Mou22}). This means that the dissipative losses include the non-convex losses that have a local optimum somewhat close to zero and losses whose tail behavior is similar to the strongly convex losses.
A typical example of losses that satisfy dissipativity is non-convex losses with $l_2$ regularization that used for many machine learning models including deep learning models (see \citet{mou18a}).
Because of its capability to handle many non-convex losses, the dissipative condition is often employed in the theoretical analysis of non-convex optimization, such as \citet{Raginsky2017} and \citet{Xu2018}.

In the Bayesian context, for example, we often use the negative log-likelihood losses, which satisfy the dissipative property if the likelihood distribution satisfies the Poincar\'{e} inequality~\cite{bakry2013analysis,vempala2019rapid}.
Poincar\'{e} inequality is applicable to a wide range of practical likelihood distributions, such as log-concave distributions, distributions obtained via bounded perturbations of Poincar\'{e}-inequality-satisfying (PIS) distributions, distributions with Gaussian convolution added to bounded losses, distributions formed by Lipschitz continuous transformations of PIS distributions, and direct sums of PIS distributions~\cite{bakry2013analysis}.
Therefore, the dissipative assumption covers many useful Bayesian models, including Bayesian deep learning models~\cite{vempala2019rapid,Mou22}.

In essence, the dissipative assumption allows for the broad treatment of not only general non-convex optimization in (deep) machine learning but also non-convex losses used in Bayesian inference and Bayesian machine learning. On the other hand, it is essential to note that thick-tailed losses, such as long-tailed $t$-distributions or Cauchy distributions, cannot be handled as a dissipative loss~\cite{Mou22}.

\section{Difference from generalization error bounds based on uniform convergence}
\label{app:diff_UC_IT}
In this section, we discuss the difference between the generalization error bounds based on the information-theoretic (IT) approach and that on the basis of the uniform convergence (UC) notion.

The generalization error bound based on the UC notion guarantees that the generalization error of all hypotheses in the algorithm’s output space simultaneously vanishes as the size of the training data increases, ensuring the convergence of generalization error. 
Furthermore, this bound asserts that within the empirical risk minimization (ERM) principle, it suffices to output any hypothesis from the class that minimizes empirical risk, and by measuring model complexities such as VC-dimension or Rademacher complexity, one can evaluate generalization performance. 
In other words, the UC-based generalization bounds offer non-trivial guarantees only when the hypothesis class utilized by the algorithm, along with its complexity, is moderately constrained.

On the other hand, deep neural networks (DNNs) are included in vast hypothesis classes where the model complexity drastically increases with the size and depth of the network. 
When applied to such models, UC-based bounds turn into a vacuous metric due to the exceedingly large model complexity.
Furthermore, bounds based on UC rely solely on the hypothesis space and are unable to leverage beneficial statistics obtained from algorithms or datasets, which sometimes results in an inability to capture the true essence of generalization performance. 
For instance, the gradient variance w.r.t.~model parameters exhibit a strong correlation with the generalization performance of deep learning~\cite{Jiang20}; however, this correlation cannot be represented within UC-type bounds (see \citet{amit2022} for details). This observation leads to the recent interest in data and algorithmic-dependent generalization bounds, such as the PAC-Bayes and IT-based generalization error bounds.

The strength of the IT-based analysis lies in its capacity to directly incorporate the algorithm- and data-dependent statistics related to the generalization performance, such as the gradient variance instead of the model complexity, into the generalization error upper bounds. Especially, the gradient variance is empirically known to have a stronger correlation with the generalization performance of DNNs~\cite{Jiang20} in comparison to statistics appearing in uniform convergence analysis contexts (e.g., VC dimension, the number of parameters $d$, and the norm of parameters). 
Although the gradient variance implicitly depends on $d$, it is widely recognized that, in practice, the gradient variance becomes reasonably small as the training proceeds~\cite{Jiang20}. 

In our bounds, such as Theorem~\ref{main_thm1_time_independent}, the generalization error bound is expressed through a quantity that reflects the stability w.r.t.~variations in the training data, which is closely related to the generalization properties~\cite{mou18a,Negrea2019,Wang2021}.
While this quantity is expressed via the expectation of gradients and thus is implicitly dependent on $d$ like the gradient variance, it is expected to decrease as training progresses and generalization performance is being enhanced~\cite{mou18a,Negrea2019}~\footnote{Note that some constants in our bounds explicitly depends on $d$ under non-convex and disspative losses. Removing this dependency is our significant future work. As can be seen in Appendix~\ref{app:relax_gauss}, however, this dependence does not occur in the convex loss case.}.
Consequently, IT-based bounds offer a sensible generalization bound even in models with significantly high complexity, such as DNNs. This is why it has gained attention in the context of SGLD's generalization analysis.

In short, the core aim of IT-based analysis is to offer practical bounds that effectively account for models with high complexity, like DNNs, by directly integrating empirically validated statistics associated with generalization obtained from datasets and an algorithm. Active discussions within the realm of IT-based analysis revolve around how to analyze the generalization performance of DNNs, which involve non-convex losses, to derive bounds that lead to an accurate understanding of generalization performance.

\section{Further discussion for convergence and dependence on dimensions}
Here, we provide further discussions on IT-based bounds including ours from the perspectives of convergence and dependence on parameter dimension especially focusing on convex losses.

\subsection{On the convergence of SGLD and our bounds in convex losses}
\label{subsec:limitation_convergence}
As shown by \citet{Shwartz09}, there is a convex problem in which the unique solution of ERM fails to generalize. This shows that any optimization algorithm executed for an infinite algorithm iteration must fall into one of two categories: either it never converges to the minimum, or it fails to generalize.
The SGLD algorithm leads to the former phenomenon because the obtained parameters via SGLD do not converge to the (local) minima under a fixed temperature parameter $\beta$ for the Gaussian noise coefficient even after infinite iterations.

SGLD rather ensures convergence to a stationary distribution, known as the Gibbs posterior distribution $\pi(\mathrm{d}w) \propto e^{-\beta F(W,S)}$, when $\beta$ remains fixed over time steps.
In essence, the trajectory of parameters via the SGLD algorithm gets closer to the minima and then explores its vicinity due to the addition of Gaussian noise to the gradient. Therefore, while convergence to a target distribution occurs, convergence to the minima itself is not achieved without controlling the noise via $\beta$.

Although SGLD does not converge to the minima, it boasts a distinct edge in its ability to explore parameters globally, even within non-convex problems, thanks to the Gaussian noise.
This property enables the evaluation of how the obtained expected loss w.r.t.~the stationary distribution deviates from that with the global minima. 
Specifically, we can evaluate this difference by factors that depend on parameter dimensions $d$, $\beta$, and the constants appearing in the assumptions for the potential function, such as dissipativity and smoothness, as elaborated in Appendix~\ref{subsec:proof_excess}.
Furthermore, we can also derive the upper bounds for the population risk and excess risk both for convex and non-convex losses (see \citet{Raginsky2017, Xu2018}).

\subsection{On dependence on parameter dimension of IT-based bounds}
Recently, \citet{livni22} has shown that every algorithm that guarantees non-trivial population loss on convex problems, must carry dimension-dependent information on the sample.
Together with our Theorem~\ref{main_thm1_time_independent}, this fact implies that, if the temperature $\beta$ is dimension independent, then SGLD will not achieve non-trivial population loss on the (convex) construction in \citet{livni22}.
Alternatively, one could choose dimension-dependent $\beta$ in SGLD but then algorithmic-independent generalization bounds can be easily (and have been) obtained via standard uniform convergence argument.

Unfortunately, removing the dependence on the parameter dimension $d$ is difficult or unavoidable even if our framework is utilized when analyzing the generalization error of discretized Langevin dynamics such as SGLD through the MI between the dataset and parameters. 
On the other hand, existing and our IT-based bounds such as Theorems~\ref{thm_existing_info} and \ref{main_thm2_continuous} are expressed by the gradient variance or the stability of the expected gradient, which implicitly depend on $d$ but could be smaller than it as training proceeds and the generalization performance is enhanced.
We refer to Appendix~\ref{app:diff_UC_IT} for an explanation of the advantages of this property in the IT-based bounds.

In order to theoretically mitigate this reliance on dimensionality, it could be imperative to explore an alternative approach to evaluating generalization that deviates from the MI between parameters and data, which forms the cornerstone of this paper.
One possible avenue is, for instance, the utilization of conditional mutual information (CMI) involving super-samples (e.g., \citet{wang23}), as highlighted in \citet{livni22}, as well as methods to quantify the MI between the learned hypothesis and dataset (hMI), instead of focusing on the parameters~\cite{harutyunyan2021}. However, the drawback of these approaches is that it becomes challenging to explicitly incorporate statistics directly obtained from algorithms, such as the gradient variance, into the understanding and evaluation of generalization despite being analyses of algorithm-dependent generalization performance. Seeking IT-based bounds that not only represent algorithm- and data-dependent statistics related to generalization performance, such as gradient variance but also theoretically eliminate dimension dependence constitutes a significant future work in the context of the IT-based generalization analysis field.

\section{Theoretical properties of SGLD and Langevin diffusion}\label{app_sgld_properties}
Here, we show some theoretical properties of SGLD under Assumptions~\ref{asm_smooth} and \ref{asm_dissipative}.

\begin{lem}[Adapted from \citet{Farghly2021}]\label{lem_origin}
Suppose that Assumptions~\ref{asm_smooth} and \ref{asm_dissipative} are satisfied.
Then, for any $z\in\mathcal{Z}$, we have
\begin{align*}
\|\nabla f(0,z)\|\leq M\sqrt{\frac{b}{m}}.
\end{align*}
\begin{proof}
    We straightforwardly obtain the above claim from Assumptions~\ref{asm_smooth} and \ref{asm_dissipative} with $w=0$.
\end{proof}
\end{lem}

\begin{lem}[Modified version from \citet{Raginsky2017}]\label{lem_grad}
Suppose that Assumption~\ref{asm_smooth} is satisfied.
Then, for any $z\in\mathcal{Z}$ and all $w \in \mathcal{W}$, we have
\begin{align*}
\|\nabla f(w,z)\|\leq M\|w\|+M\sqrt{\frac{b}{m}}.
\end{align*}
\begin{proof}
\citet{Raginsky2017} derived the upper bound of the gradient as $\|\nabla f(w,z)\|\leq M\|w\|+B$ by assuming the the following condition: $\|\nabla f(0,z)\|\leq B\ (B>0)$.
We replace the constant $B$ by $M\sqrt{\frac{b}{m}}$ based on Lemma~\ref{lem_origin}.    
\end{proof}
\end{lem}

\begin{lem}[Modified version from \citet{Xu2018}]\label{lem_sg}
Suppose that Assumptions~\ref{asm_smooth} and \ref{asm_dissipative} are satisfied.
Then, for any $z\in\mathcal{Z}$, we have
\begin{align*}
\mathbb{E}\|\nabla F_S(w)-\nabla F(w,B)\|^2\leq \frac{8(n-k)M^2(\|w\|^2+\frac{k}{m})}{k(n-1)}\coloneqq 8\delta M^2\bigg(\|w\|^2+\frac{k}{m}\bigg),
\end{align*}
where $\delta\coloneqq\frac{n-k}{k(n-1)} \in (0,1]$.
\begin{proof}
\citet{Xu2018} assumed that $\nabla F_S(w)$ is dissipative.
In contrast, we posed the dissipative assumption on $\nabla f(w,z)$ for each $z$ following \citet{Farghly2021}.
We then modified the upper bound of the stochastic gradient shown in \citet{Xu2018}.
\end{proof}
\end{lem}

\begin{lem}[Modified version from \citet{Raginsky2017} and \citet{Xu2018}]\label{lem_l2norm}
Suppose that Assumptions~\ref{asm_smooth} and \ref{asm_dissipative} are satisfied.
Let $\eta \in (0,1 \wedge \frac{m}{5M^2})$ be fixed.
Then, for any $z\in\mathcal{Z}$ and any $t \in \mathbb{N}$, we have
\begin{align*}
\mathbb{E}\|W_{t+1}\|^2\leq 
(1-2\eta m+10\eta^2 M^2)\mathbb{E}\|W_{t}\|^2+2\eta \bigg(b+10\eta M^2\frac{b}{m}+\frac{d}{\beta}\bigg),
\end{align*}
and
\begin{align}
\label{eq:w_norm}
\mathbb{E}\|W_{t+1}\|^2\leq
\begin{cases}
&2\eta (b+10\eta M^2\frac{b}{m}+\frac{d}{\beta})\quad (A_{\eta,m,M} \leq 0) \\
&(1-2\eta m+10\eta^2 M^2)^t\mathbb{E}\|W_{0}\|^2+2\frac{ b+10\eta M^2\frac{b}{m}+\frac{d}{\beta}}{m-5\eta M^2}\quad (0\leq A_{\eta,m,M} \leq 1),
\end{cases}
\end{align}
where $A_{\eta,m,M} \coloneqq (1-2\eta m+10\eta^2 M^2)$.
Combining the inequalities in Eq.~\eqref{eq:w_norm}, we have
\begin{align*}
\mathbb{E}\|W_{t}\|^2 &\leq \mathbb{E}\|W_{0}\|^2 +2\bigg(1\vee \frac{1}{m}\bigg)\bigg(b+10\eta M^2\frac{b}{m}+\frac{d}{\beta}\bigg) \notag \\
&\leq s^2+2\bigg(1\vee \frac{1}{m}\bigg)\bigg(b+10\eta M^2\frac{b}{m}+\frac{d}{\beta}\bigg) \eqqcolon C_0,
\end{align*}
where $C_0$ is independent of $\eta$ and $\beta$ and $s^2$ is the square moment of the initial distribution.
\begin{proof}
We slightly modified the coefficients of the upper bound of the $l_2$ norm of the parameter shown in \citet{Raginsky2017} and \citet{Xu2018} based on the upper bound of the stochastic gradient Lemma~\ref{lem_sg}.
\end{proof}
\end{lem}

Combining Lemmas~\ref{lem_grad} and \ref{lem_l2norm}, we have the following upper bound for the stochastic gradient.
\begin{lem}\label{lem_sg_bound}
Suppose that Assumptions~\ref{asm_smooth} and \ref{asm_dissipative} are satisfied.
Let $\eta \in (0,1 \wedge \frac{m}{5M^2})$ be fixed.
Then, for any $z\in\mathcal{Z}$ any $t \in \mathbb{N}$, we have
\begin{align*}
&\mathbb{E}\|\nabla F(W_t,B)\|^2\leq \notag \\
&\begin{cases}
&2M^2\eta^2 (b+4\eta M^2\frac{b}{m}+\frac{d}{\beta})+M^2\frac{b}{m}\quad (A_{\eta,m,M} \leq 0) \\
&M^2(1-2\eta m+10\eta^2 M^2)^{t}\mathbb{E}\|W_{0}\|^2+2M^2\frac{ b+4\eta M^2\frac{b}{m}+\frac{d}{\beta}}{m-\eta M^2}+M^2\frac{b}{m}\quad (0\leq A_{\eta,m,M} \leq 1),
\end{cases}
\end{align*}
where $A_{\eta,m,M} \coloneqq (1-2\eta m+10\eta^2 M^2)$.
Combining the above inequalities, we have
\begin{align*}
&\mathbb{E}\|\nabla F(W_t,B)\|^2\leq M^2C_0+M^2\frac{b}{m}.
\end{align*}
\end{lem}
\begin{proof}
We obtain the result by combining Lemma~\ref{lem_grad} and \ref{lem_l2norm} and using the Jensen inequality.
\end{proof}

\section{Proofs of the generalization error bound with surrogate loss in Section~\ref{sec_sgld}}
\label{app_sec_sgld}
This section provides the complete proof of Theorem~\ref{main_thm1_time_independent} restated as follows.
\maingenbound*
Our proof consists of the following three steps.
We construct the FP equations for the density of the parameters under two different datasets and derive the time evolution of the KL divergence as the upper bound of the MI (Appendix~\ref{subsec:FP_equations}).
We then analyze this time evolution by using the \emph{parametrix method} (Appendix~\ref{app_parametrix}) for solving the FP equation.
The distinction between this section and Section~\ref{subsec:proof_outline_lemma} lies in the focus of SGLD, which employs stochastic gradients and random noise from $\mathcal{N}(0,\mathbf{I}_{d})$, as opposed to the continuous Langevin diffusion that employs full-batch gradients and standard Brownian motion.

\subsection{FP equation for SGLD and time evolution of the KL divergence}
\label{subsec:FP_equations}
As the first step, we construct the two different FP equations for the parameter density.

We define a one-step SGLD at the initial step as follows:
\begin{align*}
\mathrm{d}W_{t}=-\nabla F(W_{0},B_{0})\mathrm{d}t+\sqrt{2\beta^{-1}}\mathrm{d}H_t.
\end{align*}
Note that, at time $t=\eta$,
\begin{align*}
W_{\eta}=W_0-\eta \nabla F(W_{0},B_{0})+\sqrt{2\eta\beta^{-1}}H_\eta,
\end{align*}
is distributionally equivalent to
\begin{align*}
W_{\eta}=W_0-\eta \nabla F(W_{0},B_{0})+\sqrt{2\eta\beta^{-1}}\xi,
\end{align*}
where $\xi\sim N(0,\mathbf{I}_d)$.

The distribution $\rho_t$ of $W_t$ depends on random variables $W_0$ and $B_0$. 
We thus denote the joint distribution of $\{W_0,W_t,B_0\}$ under a dataset $S$ as $\rho_{0tB}(W_0,W_t,B_0)$, where $\rho_0$ is the distribution of $W_0$ and $U$ is the \emph{uniform} distribution of $B_0$.
Then, its conditional and marginal distribution is expressed as
\begin{align}
\rho_{0tB}(W_0,W_t,B_0)=\rho_0(w_0)U(B_0)\rho_{t|0,B}(W_t|W_0,B_0)=\rho_{tB}(W_t,B)\rho_{0|t,B}(W_0|W_t,B_0).
\end{align}
Since we have introduced so many notations, for the sake of simplicity, we allow the abuse of notation and let $\rho_{t|0B}$ and $\rho_{t}$ denote both the distribution and density.

On the basis of these facts, we can obtain the FP equation for $\rho_{t|0B}(w_t)$ as
\begin{align}\label{eq:FP_rho_SGLD}
\frac{\partial \rho_{t|0B}}{\partial t}=\nabla \cdot \bigg(\frac{1}{\beta}\nabla\rho_{t|0B} + \rho_{t|0B} \nabla F(w_0,B_0)\bigg),
\end{align}
and its marginal process as
\begin{align}\label{app_FP_SGLD}
\frac{\partial \rho_{t}}{\partial t}=\nabla \cdot \bigg(\frac{1}{\beta}\nabla\rho_{t} + \rho_{t} \mathbb{E}_{\rho_{0B|t}}[\nabla F(w_0,B_0)|w_t=w]\bigg),
\end{align}
which is derived in \citet{vempala2019rapid} and \citet{kinoshita2022}.
As shown above, the randomness associated with the dataset can be handled by simply taking the expectation for conditional gradients with respect to a uniform distribution $U$. 
To avoid cumbersome discussions, we omit descriptions related to the expectation with respect to $B_{0}$ from here on.

As with the first step, we can define a one-step SGLD with the joint density $\gamma_{0t}(w_0,w_t)$ and the conditional distribution $\gamma_0(w_0)\gamma_{t|0}(w_t|w_0)$ under a dataset $S' (\neq S)$, where $\gamma_0(w_0)\gamma_{t|0}(w_t|w_0)$ corresponds to the marginal distribution, i.e., $\gamma_0(w_0)\gamma_{t|0}(w_t|w_0)=\gamma_t(w_t)\gamma_{0|t}(w_0|w_t)$.
We also can obtain the FP equation and its marginal process in the form of $\rho_{t|0}$ and $\rho_{t}$ replaced by $\gamma_{t|0}$ and $\gamma_{t}$ in Eqs.~\eqref{eq:FP_rho_SGLD} and \eqref{app_FP_SGLD}.

In Section~\ref{subsec:proof_outline_lemma}, we obtain the upper bound of the MI as follows:
\begin{align}
\label{eq:MI_upperbound}
 I(W_{t};S)\leq \mathbb{E}_{S, S'}\mathrm{KL}(P_{W_t|S}|P_{W_t|S'})=\mathbb{E}_{S, S'}\mathrm{KL}(\rho_t|\gamma_t).
\end{align}
By taking the derivation w.r.t.~$t$, we have
\begin{align}
\label{eq:KL_derivative}
\frac{\partial \mathrm{KL}(\rho_t|\gamma_t)}{\partial t}
&=\int \mathrm{d}w \left(\frac{\partial \rho_t}{\partial t}\log\frac{\rho_t}{\gamma_t}\right)-\int \mathrm{d}w \left(\frac{\rho_t}{\gamma_t}\frac{\partial \gamma_t}{\partial t}\right).
\end{align}
The first and second terms can be expressed as
\begin{align}
\int \mathrm{d} w \left(\frac{\partial \rho_t}{\partial t}\log\frac{\rho_t}{\gamma_t}\right)
=-\frac{1}{\beta}\int \mathrm{d}w\nabla\log \rho_t\cdot\nabla\log\frac{\rho_t}{\gamma_t}
-\int \mathrm{d}w\rho_{0t}\nabla\log\frac{\rho_t}{\gamma_t}\cdot\nabla F(w_0,B_0),
\end{align}
and
\begin{align}
\int \mathrm{d}w \left(\frac{\rho_t}{\gamma_t}\frac{\partial \gamma_t}{\partial t}\right) 
=-\frac{1}{\beta}\int \mathrm{d}w \nabla\frac{\rho_t}{\gamma_t}\cdot\nabla\gamma_t-\int \mathrm{d}w \frac{\rho_t}{\gamma_t}\nabla\log\frac{\rho_t}{\gamma_t}\cdot \gamma_{0t}\nabla F(W'_0,B'_0),
\end{align}
where $W'_0$ follows $\gamma_0$, which is the density of the initial distribution and $B'_0$ is the stochastic gradient based on $S'$.

According to these facts, Eq.~\eqref{eq:KL_derivative} can be rewritten as
\begin{align}\label{eq:KL_upperbound}
&\frac{\partial \mathrm{KL}(\rho_t|\gamma_t)}{\partial t} \\
&= -\frac{1}{\beta}\mathbb{E}_{\rho_t}\|\nabla \log\rho_t-\nabla \log \gamma_t \|^2 \\
&\quad\quad -\int \rho_{0t}\mathrm{d}w\nabla\log\frac{\rho_t}{\gamma_t}\cdot\nabla F(W_0,B_0) +\int \mathrm{d}w \frac{\rho_t}{\gamma_t}\nabla\log\frac{\rho_t}{\gamma_t}\cdot \gamma_{0t}\nabla F(W'_0,B'_0) \\
&= -\frac{1}{\beta}\mathbb{E}_{\rho_t}\|\nabla \log\rho_t-\nabla \log \gamma_t \|^2 -\int \rho_{t}\mathrm{d}w\nabla\log\frac{\rho_t}{\gamma_t}\cdot \mathbb{E}_{\rho_{0|t}}[\nabla F(w_0,B_0)|W_t=w]\\
&\quad\quad +\int \mathrm{d}w\frac{\rho_t}{\gamma_t}\nabla\log\frac{\rho_t}{\gamma_t}\cdot \gamma_{t} \mathbb{E}_{\gamma_{0|t}}[\nabla F(w'_0,B'_t)|W_t=w)]\\
&= -\frac{1}{\beta}\mathbb{E}_{\rho_t}\|\nabla \log\rho_t-\nabla \log \gamma_t \|^2 \\
&\quad \quad - \int \rho_{t}\mathrm{d}w\nabla\log\frac{\rho_t}{\gamma_t}\cdot (\mathbb{E}_{\rho_{0|t}}[\nabla F(W_0,B_0)|W_t=w] - \mathbb{E}_{\gamma_{0|t}}[\nabla F(W'_0,B'_t)|W_t=w])\\
&\leq -\frac{1}{2\beta}\mathbb{E}_{\rho_t}\|\nabla \log\rho_t-\nabla \log \gamma_t \|^2 \\
&\quad \quad +\frac{\beta}{2} \int \rho_{t}\mathrm{d}w\|\mathbb{E}_{\rho_{0|t}}[\nabla F(W_0,B_0)|W_t=w] -\mathbb{E}_{\gamma_{0|t}}[\nabla F(W'_0,B'_t)|W_t=w]\|^2,
\end{align}
where the final inequality comes from the Cauchy--Schwartz inequality. We define $\widetilde{V}_{\nabla_t} \coloneqq \int \rho_{t}\mathrm{d}w\|\mathbb{E}_{\rho_{0|t}}[\nabla F(W_0,B_0)|W_t=w] -\mathbb{E}_{\gamma_{0|t}}[\nabla F(W'_0,B'_t)|W_t=w]\|^2$ for simplicity. We evaluate this term in Appendix~\ref{app:stability_term}.

In the same way as Section~\ref{subsec:proof_outline_lemma}, we have the following inequality by introducing the logarithm of the stationary distribution $\nabla \log \pi(w)$ into $\mathbb{E}\|\nabla \log\rho_t-\nabla\log \gamma_t\|^2$ in the above:
\begin{align}
&-\mathbb{E}_{\rho_t}\|\nabla \log\rho_t-\nabla \log \gamma_t \|^2 \\
&=-\mathbb{E}_{\rho_t}\|\nabla \log\rho_t\|^2-\mathbb{E}_{\rho_t}\|\nabla \log\gamma_t\|^2+2\mathbb{E}_{\rho_t}\nabla \log\rho_t \cdot \nabla \log\gamma_t \\
&\leq-\frac{1}{2}\mathbb{E}_{\rho_t}\|\nabla \log\rho_t-\nabla \log \pi\|^2+\mathbb{E}_{\rho_t}\|\nabla \log \pi\|^2+2\mathbb{E}_{\rho_t}\nabla \log\rho_t \cdot \nabla \log\gamma_t,
\end{align}
where we used $-x^2\leq -\|x-y\|^2/2+y^2$ for all $x, y\in\mathbb{R}^d$. To simplify the notation, we express the second and third terms as $\Omega(\rho_t,\gamma_t,\pi)\coloneqq \mathbb{E}_{\rho_t}\|\nabla \log \pi\|^2+2\mathbb{E}_{\rho_t}\nabla \log\rho_t \cdot \nabla \log\gamma_t$.
From the above fact, we obtain
\begin{align}
\label{eq:kl_time_upperbound}
\frac{\partial \mathrm{KL}(\rho_t|\gamma_t)}{\partial t}
&\leq -\frac{1}{4\beta}\mathbb{E}_{\rho_t}\|\nabla \log\rho_t-\nabla \log \pi \|^2 +\frac{1}{2\beta}\Omega(\rho_t,\gamma_t,\pi) +\frac{\beta}{2} \widetilde{V}_{\nabla_t} \notag \\
&\leq -\frac{1}{4\beta c_{\mathrm{LS}}} \mathrm{KL}(\rho_t|\pi) +\frac{1}{2\beta}\Omega(\rho_t,\gamma_t,\pi) +\frac{\beta}{2} \widetilde{V}_{\nabla_t} \notag\\
&\leq -\frac{1}{4\beta c_{\mathrm{LS}}} \left(\mathrm{KL}(\rho_t|\gamma_t)+\mathbb{E}_{\rho_t}\log\frac{\gamma_t}{\pi}\right) +\frac{1}{2\beta}\Omega(\rho_t,\gamma_t,\pi) +\frac{\beta}{2} \widetilde{V}_{\nabla_t},
\end{align}
where the second inequality is from the LSI~\cite{bakry2013analysis}. 
In the above, the LSI constant $c_{\mathrm{LS}}$ is defined by~\citet{bakry2013analysis} as follows:
\begin{align}
&c_\mathrm{LS}\leq \lambda_{l} \coloneqq 2D_{1}+2\rho_{0}^{-1}(D_{2}+2), \label{main:LSICONSTNANT} \\
&\rho_0^{-1}\leq\frac{2C(d+b\beta)}{m\beta} \exp \left(\frac{2}{m}(M+B)(b\beta+d)+\beta(A+B)\right)+\frac{1}{m\beta(d+b\beta)}, \label{main:rho}
\end{align}
where $D_1=\frac{2m^2+8M^2}{\beta m^2M}$, $D_2\leq\frac{6M(d+\beta)}{m}$, and $C$ is the universal constant (see also Appendices~B and E in \citet{Raginsky2017}).

Multiplying $e^{\frac{t}{4\beta c_{\mathrm{LS}}}}$ for both hands in Eq.~\eqref{eq:kl_time_upperbound} yields
\begin{align}
&e^{\frac{t}{4\beta c_{\mathrm{LS}}}} \frac{\partial \mathrm{KL}(\rho_{t}|\gamma_{t})}{\partial t}\leq  -e^{\frac{t}{4\beta c_{\mathrm{LS}}}}\frac{1}{4\beta c_{\mathrm{LS}}}\mathbb{E}_{\rho_{t}}\log\frac{\gamma_{t}}{\pi} +\frac{1}{2\beta}e^{\frac{t}{4\beta c_{\mathrm{LS}}}}\Omega(\rho_t,\gamma_t,\pi)+\frac{\beta}{2} \widetilde{V}_{\nabla_t}.
\end{align}
By integrating time $t=0\to \eta$, we obtain
\begin{align}\label{upper_bound}
&e^{\frac{\eta}{4\beta c_{\mathrm{LS}}}}\mathrm{KL}(\rho_{\eta}|\gamma_{\eta})\leq  \mathrm{KL}(\rho_{0}|\gamma_{0})-\int_0^\eta \mathrm{d}t e^{\frac{t}{4\beta c_{\mathrm{LS}}}}\frac{1}{4\beta c_{\mathrm{LS}}}\mathbb{E}_{\rho_{t}}\log\frac{\gamma_{t}}{\pi}\notag \\
&\quad\quad +\int_0^\eta \mathrm{d}t \frac{1}{2\beta}e^{\frac{t}{4\beta c_{\mathrm{LS}}}}\Omega(\rho_t,\gamma_t,\pi)+\frac{\beta}{2} \int_0^\eta \mathrm{d}t e^{\frac{t}{4\beta c_{\mathrm{LS}}}}\widetilde{V}_{\nabla_t}.
\end{align}

We later evaluate the second and third terms of Eq.~\eqref{upper_bound} in Appendix~\ref{app_parametrix}. We then evaluate the fourth term $\widetilde{V}_{\nabla_t}$ of Eq.~\eqref{upper_bound} as $\widetilde{V}_{\nabla_t} \leq D_1$, where $D_1$ is problem dependent constant that can be independent of $\eta$, in Appendix~\ref{app:stability_term}.

\subsection{The solution for the FP equation via the parametrix method}
\label{app_parametrix}
In this section, we evaluate the second and third terms in Eq.~\eqref{upper_bound} by utilizing the \emph{parametrix method} for the FP equation~\cite{friedman2008partial, pavliotis2014stochastic}.
\subsubsection{Consequences of the parametrix method for the FP equation}
\label{subsubsec:conseq_parametrix}
We summarize two essential consequences of the parametrix method used in Appendix~\ref{subsubsec:param_method_apply}.

\paragraph{Solution expansion of the FP equation.}
The first consequence of this method is as follows.
Given the following FP equation,
\begin{align}
 \nabla \cdot \bigg(\frac{1}{\beta}\nabla\rho_t(w) + \rho_t(w) b(w,t)\bigg)-\frac{\partial \rho_t(w)}{\partial t}=0,
\end{align}
and initial condition $\rho_{t=0}=\rho_0$, the solution can be expanded as
\begin{align}\label{eq_parametrix_fundamental_solution}
\rho_t(w)=\mathbb{E}_{\xi\sim \rho_0}Z(w,t;\xi,0)+\mathbb{E}_{\xi\sim \rho_0}\int_{0}^t\int_{\mathbb{R}^d} \mathrm{d}z Z(w,t;z,\tau)\Phi(z,\tau;\xi,0),
\end{align}
where 
\begin{align}
Z(x,t;z,\tau)\coloneqq\frac{1}{(\frac{4\pi}{\beta} (t-\tau))^{d/2}}e^{-\frac{\beta\|x-z\|^2}{4(t-\tau)}},
\end{align}
and
\begin{align}\label{eq_expand_PL}
\Phi(z,\tau;\xi,t)\coloneqq\sum_{n=1}^\infty L^{n}Z (z,\tau;\xi,t).
\end{align}
In Eq.~\eqref{eq_expand_PL}, $L^{n}Z (z,\tau;\xi,t)$ is defined through
\begin{align}
L^{n+1}Z(x,\tau;\xi,t)\coloneqq\int_t^\tau \int_{\mathbb{R}^d} \mathrm{d}s\mathrm{d}y (LZ(x,\tau;y,s))(L^{n}Z(y,s;\xi,t)),
\end{align}
where 
\begin{align}
LZ(x,\tau;\xi,t)\coloneqq b(x,\tau)\cdot \nabla_x Z(x,\tau;\xi,t),
\end{align}
and thus $L^{1}Z(x,\tau;\xi,t)=LZ(x,\tau;\xi,t)$. 

The above expansion requires the convergence of Eq.~\eqref{eq_expand_PL}.
Fortunately, this condition holds for the Langevin diffusion, for example, because $\rho_{0}(w)$ and $b(w,t) = \nabla F(w, S)$ satisfies the following two assumptions for the initial state and $b(w,t)$ from Lemma~\ref{lem_grad} and Assumption~\ref{regularity_FP}: (i) there exist positive constants $a$ and $b$ such that $\rho_0(w)\leq ae^{b\|w\|^2}<\infty$ for all $w\in\mathcal{W}$, and (ii) there exist some positive constants $a'$ and $b'$ such that $\|b(w,t)\|<a'\|w\|+b'$ for all $w\in\mathcal{W}$.

\paragraph{Parametrix solution is twice differentiable.}
Another important consequence is that the parametrix solution $\rho_t(w)$ is twice differentiable with respect to $w$.
Under the initial distribution $\mathcal{N}(0,s^2\mathbf{I}_d)$ with Assumption~\ref{regularity_FP}, we can obtain the following facts according to \citet{friedman2008partial} and \citet{pavliotis2014stochastic}:
\begin{align}
&\rho_t(w) \leq \frac{1}{(2\pi(s^2+\frac{2t}{\beta}))^{d/2}}e^{-\frac{\|w\|^2}{2(s^2+\frac{2t}{\beta})}}+\bigg(s^2+\frac{2t}{\beta}\bigg)^{1/2}\frac{C_0}{(2\pi(s^2+\frac{2t}{\beta}))^{d/2}}e^{-\frac{\|w\|^2}{2(s^2+\frac{2t}{\beta})}}\label{app_parametrix_rho},\\
&\sum_{i=1}^d\left|\frac{\partial\rho_t(w)}{\partial w_i}\right|  \leq \frac{C_1}{(2\pi(s^2+\frac{2t}{\beta}))^{(d+1)/2}}e^{-\frac{\|w\|^2}{2(s^2+\frac{2t}{\beta})}}+\frac{C_2}{(2\pi(s^2+\frac{2t}{\beta}))^{d/2}}e^{-\frac{\|w\|^2}{2(s^2+\frac{2t}{\beta})}}\label{app_parametrix_rho_x}, \\
\end{align}
and
\begin{align}
\sum_{i,j=1}^d\left|\frac{\partial^2\rho_t(w)}{\partial w_i\partial w_i}\right| &\leq \frac{C_3}{(2\pi(s^2+\frac{2t}{\beta}))^{(d+2)/2}}e^{-\frac{\|w\|^2}{2(s^2+\frac{2t}{\beta})}} \\
&\ \ \ \ \ \ \ \ \ \ \ \ \ \ \ \ \ \ \ \ \ \ \ \ \ \ \ \ \ \ \ \ \ \ \ \ \ \ +\bigg(s^2+\frac{2t}{\beta}\bigg)^{-1/2}\frac{C_4}{(2\pi(s^2+\frac{2t}{\beta}))^{d/2}}e^{-\frac{\|w\|^2}{2(s^2+\frac{2t}{\beta})}}\label{app_parametrix_rho_xx},
\end{align}
where $\{C_0, C_1, C_2, C_3, C_4\}$ are positive constants w.r.t.~$\{m,M,\beta,d,b,s^2\}$.

Eqs.~\eqref{app_parametrix_rho} and \eqref{app_parametrix_rho_x} can be derived by following the proof of Theorem~11 in \citet{friedman2008partial}.
The statement of this theorem is about the transition kernel; therefore, it corresponds to the case where the expectation with respect to $\xi\sim\rho_0$ is excluded from Eq.~\eqref{eq_parametrix_fundamental_solution}.
In light of this fact, we take the convolution by the initial distribution $\xi\sim\rho_0$ for the beginning part of the proof of Theorem~11.
This approach is equivalent to taking convolutions in the overall discussion of Section~4 described by \citet{friedman2008partial}.
After convolution, following the proof of Theorem~11 leads to Eqs.~\eqref{app_parametrix_rho} and \eqref{app_parametrix_rho_x}.
Theorem~11 does not yield results related to second-order differentials; however, we can derive Eq.~\eqref{app_parametrix_rho_xx} by combining Lemma~3 from \citet{friedman2008partial} into the proof of Theorem~11 and employing a similar way as described above.
While it is assumed that the coefficients of the FP equation are bounded in \citet{friedman2008partial}, we can relax this assumption to unbounded coefficients as shown in subsequent work such as \citet{deck2002parabolic}.

\subsubsection{Applying the parametrix method for SGLD's FP equations}
\label{subsubsec:param_method_apply}
Now, we get back to the SGLD setting.
First of all, it should be mentioned that we do not lose generality by focusing solely on the initial iteration, i.e., $t=0\to \eta$. 
This reason is as follows.

For the initial iteration ($t=0 \to \eta$), we can see that $b(x,t)=\mathbb{E}_{\rho_{0|t}}[\nabla F(W_0,B_0)|W_t=w]$ from Eq.~\eqref{app_FP_SGLD}.
As we explained in Appendix~\ref{subsubsec:conseq_parametrix}, the condition of the expansion is satisfied under Lemma~\ref{lem_grad} and Assumption~\ref{regularity_FP}.
Thus, the solution and its differentiation can be obtained via the parametrix solution, expressed as Eqs.~\eqref{app_parametrix_rho}, \eqref{app_parametrix_rho_x}, and \eqref{app_parametrix_rho_xx} with the constants $\{C_0, C_1, C_2, C_3, C_4\}$ that depend on the problem except $\eta$.
When considering the second iteration ($t=\eta \to 2\eta$), the initial distribution is expressed as $\rho_\eta$.
The concern here is whether the solution of the FP for SGLD satisfies the conditions of the parametrix method in this case.
Fortunately, these conditions are also satisfied in the second iteration.
The initial condition of the expansion is satisfied from Eq.~\eqref{app_parametrix_rho}, and the condition $b(x,t)=\mathbb{E}_{\rho_{\eta|t}}[\nabla F(W_\eta,B_1)|W_t=w]$ also satisfies the condition of the FP expansion from Lemma~\ref{lem_grad}. 
We thus have the same form of the solution in Eqs.~\eqref{app_parametrix_rho}, \eqref{app_parametrix_rho_x}, and \eqref{app_parametrix_rho_xx} at time $t=\eta\to 2\eta$.
In the same way, the solution at $t\in (s\eta,(s+1)\eta]$ for $s\in\mathbb{N}$ can be expanded as the same parametrix expansion.

\paragraph{Bounding $\Omega(\rho_t,\gamma_t,\pi)$ (related to the third term in Eq.~\eqref{upper_bound}).}
We can decompose $\Omega(\rho_t,\gamma_t,\pi)$ as
\begin{align}
\mathbb{E}_{\rho_t}\|\nabla \log \pi\|^2+2\mathbb{E}_{\rho_t}\nabla \log\rho_t \cdot \nabla \log\gamma_t
= \mathbb{E}_{\rho_t}\|\nabla \log \pi\|^2+2\int \mathrm{d}w \nabla \rho_t \cdot \nabla \log\gamma_t.
\end{align}
To derive the upper bound of the right-hand side, we focus on the following facts:
\begin{align}\label{app_partial_derivative}
\int \mathrm{d}w \nabla \rho_t \cdot \nabla \log\gamma_t=\int \mathrm{d}w \sum_{i=1}^d\frac{\partial \rho_t}{\partial w_i}\frac{\partial \log \gamma_t}{\partial w_i}.
\end{align}
For the $i$-th dimension, we have
\begin{align}
\label{eq:bound_twice_diff}
\int \mathrm{d}w \frac{\partial \rho_t}{\partial w_i}\frac{\partial \log \gamma_t}{\partial w_i}= - \int \mathrm{d}w \frac{\partial^2 \rho_t}{\partial w_i^2}\log \gamma_t \leq \int \mathrm{d}w \left|\frac{\partial^2 \rho_t}{\partial w_i^2}\right|\left|\log \gamma_t\right|, 
\end{align}
where we used the integration by parts from the fact that $\frac{\partial \rho_t}{\partial w_i}\to 0$ as $\|w\|\to 0$ according to the expansion in Eq.~\eqref{app_parametrix_rho_x}.
Details of this argument can be found in \citet{Mou22}.
By using the result in Eq.~\eqref{app_parametrix_rho} for $\rho_t$ and $\gamma_t$, we have
\begin{align}
\label{eq:bound_absolute}
\int \mathrm{d}w \left|\frac{\partial^2 \rho_t}{\partial w_i^2}\right| \left|\log \gamma_t \right| \leq \frac{C'_1}{(2\pi(s^2+\frac{2t}{\beta}))^{1/2}}+C_2',
\end{align}
where the Gaussian integral is used for $\frac{\partial^2 \rho_t}{\partial w_i^2}$ in Eq.~\eqref{app_parametrix_rho_xx} and $\rho_{t}$ in Eq.~\eqref{app_parametrix_rho} is replaced to $\gamma_t$.
We note that $C'_1$ and $C'_2$ only depend on $\{m,M,\beta,d,b,s^2\}$.

Substituting Eqs.~\eqref{eq:bound_twice_diff} and \eqref{eq:bound_absolute} into Eq.~\eqref{app_partial_derivative}, we obtain
\begin{align}\label{app_partial_derivative_upper}
\int \mathrm{d}w \nabla \rho_t \cdot \nabla \log\gamma_t\leq \frac{dC'_1}{(2\pi(s^2+\frac{2t}{\beta}))^{1/2}}+dC_2'.
\end{align}
From Lemma~\ref{lem_sg_bound}, we have
\begin{align}
\label{eq:log_pi_upper}
    \mathbb{E}_{\rho_t}\|\nabla \log \pi\|^2 &\leq \beta^2 \mathbb{E}_{\rho_t}\|\nabla F_S(w)\|^2 \notag \\
    &\leq \beta^2M^2\bigg(s^2+2\bigg(1\vee \frac{1}{m}\bigg)\bigg(b+10 M^2\frac{b}{m}+\frac{d}{\beta}\bigg) \bigg)+\beta^2M^2\frac{b}{m}.
\end{align}

Eqs.~\eqref{app_partial_derivative_upper} and \eqref{eq:log_pi_upper} leads to $\Omega(\rho_t,\gamma_t,\pi)\leq D_2$, where 
\begin{align}\label{app_omega}
D_2 \coloneqq \beta^2M^2\bigg(s^2+2\bigg(1\vee \frac{1}{m}\bigg)\bigg(b+10 M^2\frac{b}{m}+\frac{d}{\beta}\bigg) + \frac{b}{m} \bigg)
+ \frac{dC'_1}{(2\pi(s^2+\frac{2t}{\beta}))^{1/2}}+dC_2'.
\end{align}

\paragraph{Bounding $-\mathbb{E}_{\rho_{t}}\log\frac{\gamma_{t}}{\pi}$ (related to the second term in Eq.~\eqref{upper_bound}).}
By using the Kolmogorov solution of the FP equation~\cite{bakry2013analysis}, we have
\begin{align}\label{eq:upper_conditioned_expectation}
\gamma_t(w) &= \mathbb{E}_{W_T}[\gamma_0(W_T)|W_0=w] \notag \\
&= \mathbb{E}_{W_T}\left[\frac{1}{(2\pi s^{2})^{d/2}}e^{-\frac{\|W_T\|^2}{2s^2}}\bigg|W_0=w\right] \geq \frac{1}{(2\pi s^2)^{d/2}}e^{-\frac{\mathbb{E}_{W_T}[\|W_T\|^2|W_0=w]}{2s^2}},
\end{align}
where the last inequality comes from Jensen's inequality.
The above inequality gives us
\begin{align}
-\mathbb{E}_{\tilde{W}_T}\log \gamma_t(\tilde{W}_T)
&=-\mathbb{E}_{\tilde{W}_T}\log \mathbb{E}_{W_T}[\gamma_0(W_T)|W_0 = \tilde{W}_T]\\
&\leq -\mathbb{E}_{\tilde{W}_T}\mathbb{E}_{W_T}[\log\gamma_0(W_T)|W_0 = \tilde{W}_T]\\
&\leq \frac{d}{2}\log (2\pi s^2)+\frac{1}{2s^2}\mathbb{E}_{\tilde{W}_T}\mathbb{E}_{W_T}[\|W_T\|^2|W_0 = \tilde{W}_T],
\end{align}
where $\tilde{W}_T$ is the independent copy of $W_{T}$ and the first and second inequalities are obtained from Jensen's inequality and Eq.~\eqref{eq:upper_conditioned_expectation}, respectively.
By using Lemma~\ref{lem_l2norm} twice, we obtain 
\begin{align*}
-\mathbb{E}_{\tilde{W}_T}\log \gamma_t(\tilde{W}_T)
&\leq \frac{d}{2}\log (2\pi s^2)+\frac{1}{2s^2}\bigg(\mathbb{E}_{\tilde{W}_T}\|\tilde{W}_T\|^2+2\bigg(1\vee \frac{1}{m}\bigg)\bigg(b+10M^2\frac{b}{m}+\frac{d}{\beta}\bigg) \bigg)\\
&\leq \underbrace{\frac{d}{2}\log (2\pi s^2)+\frac{1}{2s^2}\bigg(s^2+4\bigg(1\vee \frac{1}{m}\bigg)\bigg(b+10M^2\frac{b}{m}+\frac{d}{\beta}\bigg) \bigg)}_{\eqqcolon B_1}.
\end{align*}
In addition, from Lemma~\ref{lem_function_bound}, we have
\begin{align}
\mathbb{E}_{\rho_{t}}\log\pi
&=\beta M\mathbb{E}_{\rho_t}\|W\|^2+\frac{\beta b}{2m}+A \\
&\leq \underbrace{\beta M\bigg(s^2+2\bigg(1\vee \frac{1}{m}\bigg)\bigg(b+10M^2\frac{b}{m}+\frac{d}{\beta}\bigg) \bigg)+\frac{\beta b}{2m}+A}_{\eqqcolon B_2}.
\end{align}
Thus, we have the upper bound of $-\mathbb{E}_{\rho_{t}}\log\frac{\gamma_{t}}{\pi}$ as 
\begin{align}\label{app_kl_gamma}
-\mathbb{E}_{\rho_t}\log\frac{\gamma_t}{\pi}\leq B_{1} + B_{2} \eqqcolon D_3,
\end{align}
where $D_3$ is the positive constant only depends on $\{m,M,\beta,d,b,s^2\}$.

\subsubsection{Bounding the stability term and finalizing proof}\label{app:stability_term}
Finally, we show the upper bound of the stability term expressed as $\widetilde{V}_{\nabla_t}$ in Eq.~\eqref{upper_bound}. 
Similarly to Appendix~\ref{subsubsec:param_method_apply}, we focus on the initial iteration $t=0\to \eta$. 
From the definition of $\widetilde{V}_{\nabla_t}$, we have
\begin{align}
\widetilde{V}_{\nabla_t} &= \int \rho_{t}\mathrm{d}w\|\mathbb{E}_{\rho_{0|t}}[\nabla F(W_0,B_0)|W_t=w] -\mathbb{E}_{\gamma_{0|t}}[\nabla F(W'_0,B'_t)|W_t=w]\|^2\notag \\
&\leq 2\int \rho_{t}\mathrm{d}w\|\mathbb{E}_{\rho_{0|t}}[\nabla F(W_0,B_0)|W_t=w]\|^2+2\int \rho_{t}\mathrm{d}w\|\mathbb{E}_{\gamma_{0|t}}[\nabla F(W'_0,B'_t)|W_t=w]\|^2.
\end{align}
The first term of the above can be rewritten as
\begin{align}
\int \rho_{t}\mathrm{d}w\|\mathbb{E}_{\rho_{0|t}}[\nabla F(W_0,B_0)|W_t=w]\|^2\leq \mathbb{E}_{\rho_0}\|\nabla F(W_0,B_0)\|^2,
\end{align}
by using Jensen's inequality for the conditional distribution.
Since $\mathbb{E}_{\rho_0}\|\nabla F(W_0,B_0)\|^2$ can be bounded by using Eq.~\eqref{eq:log_pi_upper}, we have
\begin{align}\label{eq_conditional_first}
    &\int \rho_{t}\mathrm{d}w\|\mathbb{E}_{\rho_{0|t}}[\nabla F(W_0,B_0)|W_t=w]\|^2\notag \\
    &\leq M^2\bigg(s^2+2\bigg(1\vee \frac{1}{m}\bigg)\bigg(b+10 M^2\frac{b}{m}+\frac{d}{\beta}\bigg) \bigg)+M^2\frac{b}{m}\eqqcolon D_4.
\end{align}

Next, we derive the upper bound of
\begin{align}
\int \rho_{t}\mathrm{d}w\|\mathbb{E}_{\gamma_{0|t}}[\nabla F(W'_0,B'_t)|W_t=w]\|^2 = \mathbb{E}_{\rho_{t}}\mathbb{E}_{\gamma_{0|t}}[\|\nabla F(W'_0,B'_t)|W_t=w]\|^2.
\end{align}
From Lemma~\ref{lem_grad}, we have $\|\nabla F(W'_0,B'_t)\|^2\leq 2M^2\|W'_0\|^2+2M^2\frac{b}{m}$. 
Thus, we need to evaluate $\mathbb{E}_{\rho_{t}}\mathbb{E}_{\gamma_{0|t}}[\|W'_0\|^2|W'_t=w]$; however, it is difficult to analyze this expectation because the densities $\rho$ and $\gamma$ at time $t$ and $0$ are different.

Fortunately, we can circumvent this difficulty by using the reverse process formulae shown in \citet{Haussmann1986}.
According to the fact that the conditional expectation $\mathbb{E}_{\gamma_{0|t}}[\cdot]$ implies the reverse process of Eq.~\eqref{eq:FP_rho_SGLD}.
 This formulae gives us the following reverse process for time $s \ (0 \leq s \leq t)$:
\begin{align}\label{eq_reverse_process}
\mathrm{d}\tilde{W}_s=[\nabla F(\tilde{W}_{s=t},B'_{s=t})+2\beta^{-1}\nabla\log \gamma_{t-s}]\mathrm{d}t+\sqrt{2\beta^{-1}}\mathrm{d}H_s, \quad \tilde{W}_0\sim \gamma_t.
\end{align}
In the above, $B'_{s=t}$ implies $B'_0$ in the original forward process and thus a mini-batch sample is fixed. 
We obtain the relationship $\tilde{\gamma}_s=\gamma_{t-s}$, where $\tilde{\gamma}_s$ is the distribution of $\tilde{W}_s$. 
This relationship reflects the inverse process of $\gamma_t$, and we also have $\tilde{\gamma}_t=\gamma_0$ and $\tilde{\gamma}_0=\gamma_t$.
We can analyze Eq.~\eqref{eq_reverse_process} by using the parametrix method~\cite{deck2002parabolic}.
We refer to Remark~\ref{rem:parametrix_assump} for the explanation that Eq.~\eqref{eq_reverse_process} satisfies the assumptions of the parametrix method~\cite{deck2002parabolic}.

Let us express $p_\gamma (y,s|x,s')$ as the transition kernel of Eq.~\eqref{eq_reverse_process}. 
For simplicity, we express the conditional distribution given $\tilde{W_0}\sim \tilde{\gamma}_0$ as $\tilde{\gamma}_{s=t|s=0}$, which corresponds to the above transition kernel: $\tilde{\gamma}_{s=t|s=0}(y)=p_\gamma (y,s|x=w,s'=0)$.
By fixing $\tilde{W_0}$ as $w$,
we obtain $\mathbb{E}_{\gamma_{0|t}}[\|W'_0\|^2|W'_t=w]=\mathbb{E}_{\tilde{\gamma}_{s=t|s=0}}[\|\tilde{W}_t\|^2|\tilde{W_0}=w]$.
By analyzing the reverse process of $\gamma_{t}$, we can evaluate the second term in the upper bound of $\widetilde{V}_{\nabla_t}$, i.e., $\int \rho_{t}\mathrm{d}w\|\mathbb{E}_{\gamma_{0|t}}[\nabla F(W'_0,B'_t)|W_t=w]\|^2$.

We consider approximating $\tilde{\gamma}_{s=t|s=0}$ by the parametrix method to derive the upper bound of $\mathbb{E}_{\tilde{\gamma}_{s=t|s=0}}[\|\tilde{W}_t\|^2|\tilde{W_0}=w]$. 
By using the upper bound of the transition kernel provided by the parametrix method in \citet{deck2002parabolic}, we obtain
\begin{align}
p_\gamma (y,s|x,s')\leq K_1(s-s')^{-d/2}e^{-K_2\frac{\|y-x\|^2}{s-s'}},
\end{align}
where $K_1$ and $K_2$ are positive and problem-dependent constants and do not depend on $s-s'$.
From the above inequality, by setting $x=w$ and $s-s'=t$, we have
\begin{align}
\mathbb{E}_{\gamma_{0|t}}[\|W_0\|^2|W_t=w]\leq \tilde{C_0}(\|w\|^2+\tilde{C_1}t^2),
\end{align}
where $\tilde{C_0}$ and $\tilde{C_1}$ are positive and problem-dependent constants.
Thus, we have
\begin{align}\label{eq_conditional_2nd}
&\mathbb{E}_{\rho_{t}}\mathbb{E}_{\gamma_{0|t}}[\|\nabla F(W'_0,B'_t)|W_t=w]\|^2\notag \\
&\leq 2M^2\tilde{C_0}\bigg(\tilde{C_1}t^2+s^2+2\bigg(1\vee \frac{1}{m}\bigg)\bigg(b+10M^2\frac{b}{m}+\frac{d}{\beta}\bigg)\bigg)+2M^2\frac{b}{m}\eqqcolon D_5.
\end{align}
In the above, $\tilde{C_1}t^2$ is negligibly much smaller than the other terms within $0\leq t\leq \eta$.

From Eq.~\eqref{eq_conditional_first} and Eq.~\eqref{eq_conditional_2nd}, we obtain
\begin{align}
\label{eq:stability_bound_const}
\widetilde{V}_{\nabla_t}\leq  2(D_4+D_5)\coloneqq D_1.
\end{align}

We conclude this section by finalizing the proof of Theorem~\ref{main_thm1_time_independent}.
By combining Eqs.~\eqref{app_omega}, \eqref{app_kl_gamma}, and \eqref{eq:stability_bound_const} with Eq.~\eqref{upper_bound} and taking the expectation with respect to all of the randomness, we obtain
\begin{align}
\mathbb{E}_{S,S'}\mathrm{KL}(\rho_{\eta}|\gamma_{\eta}) \leq e^{\frac{-\eta}{4\beta c_{\mathrm{LS}}}}\mathbb{E}_{S,S'}\mathrm{KL}(\rho_{0}|\gamma_{0})&+(1-e^{\frac{-\eta}{4\beta c_{\mathrm{LS}}}})D_2 \\
&+2c_{\mathrm{LS}}(1-e^{\frac{-\eta}{4\beta c_{\mathrm{LS}}}}) D_3
+2\beta^2 c_{\mathrm{LS}}(1-e^{\frac{-\eta}{4\beta c_{\mathrm{LS}}}})D_1,
\end{align}
where we used the following fact:
\begin{align}
    \int_0^\eta \mathrm{d}t e^{\frac{t}{4\beta c_{\mathrm{LS}}}}=4\beta c_{\mathrm{LS}} (e^{\frac{\eta}{4\beta c_{\mathrm{LS}}}}-1).
\end{align}
Since $e^{\frac{-\eta}{4\beta c_{\mathrm{LS}}}}\geq 1-\frac{\eta}{4\beta c_{\mathrm{LS}}}$ from the assumption, we have
\begin{align}\label{app_KL_upper_bound}
&\mathbb{E}_{S,S'}\mathrm{KL}(\rho_{\eta}|\gamma_{\eta})\leq  e^{\frac{-\eta}{4\beta c_{\mathrm{LS}}}}\mathbb{E}_{S,S'}\mathrm{KL}(\rho_{0}|\gamma_{0})+\frac{\eta}{4\beta c_{\mathrm{LS}}}D_2+\frac{\eta}{2\beta } D_3
+\frac{\eta\beta}{2}D_1.
\end{align}
This concludes the proof.

\begin{rem}
\label{rem:parametrix_assump}
We show that Eq.~\eqref{eq_reverse_process} satisfies the assumption of the parametrix method~\cite{deck2002parabolic}.
First, \citet{deck2002parabolic} assumes the strong regularity condition for the diffusion coefficient, which is satisfied because the diffusion coefficient in our setting is a constant.
Next, we confirm the assumptions that the drift coefficient $b(w,s)\coloneqq \nabla F(\tilde{W}_{s=t},B'_{s=t})+2\beta^{-1}\nabla\log \gamma_{t-s}$ must satisfy.
Specifically, the following two assumptions for $b(w,s)$ must be satisfied: (i) the locally H\"{o}lder continuous condition on some bounded subset in $\mathbb{R}^d$ and (ii) the global growing condition, that is, $\|b(w,s)\|\leq c_0(\|x\|+1)$ with some positive constant $c_0$.
Fortunately, for $\nabla F$ in $b(w,s)$, the assumption (i) is satisfied by Assumption~\ref{asm_smooth}, and the assumption (ii) holds from Lemma~\ref{lem_grad}. 
Furthermore, $\nabla\log \gamma_{t-s}$ in $b(w,s)$ also satisfies the assumption (ii) from Lemma~E.1 in \citet{Mou22}.
According to the fact that $\nabla\gamma_{t-s}$ satisfies the H\"{o}lder continuous as shown in \cite{friedman2008partial}, we can see that $\frac{1}{\gamma_{t-s}}$ is bounded by considering the bounded set in $\mathbb{R}^d$.
This means that $\nabla\log \gamma_{t-s}=\frac{1}{\gamma_{t-s}}\nabla \gamma_{t-s}$ in $b(w,s)$ is H\"{o}lder continuous and satisfies the assumption (i).
\end{rem}

\subsection{On relaxing Gaussian condition in Assumption~\ref{regularity_FP}}
\label{app:relax_gauss} 
The Gaussian initial distribution assumption for $W_{0}$ could be relaxed.
Let us consider the case when the initial distribution $P_{W_{0}}$ is a mixture of Gaussian distribution, where each component of $P_{W_{0}}$ satisfies Assumption~\ref{regularity_FP}. 
In our original proof, the Gaussian assumption is used when deriving the upper bound of the finite second moment at the initial state, and when analytically marginalizing out the initial state of the transition kernel given by the fundamental solution of the parametrix method.
Even when using the mixture of Gaussian distribution as the initial distribution, it is possible to satisfy these conditions. 
The finite second-moment condition can easily be satisfied and the integration of the transition kernel can be executed by focusing on each component of the mixture distribution. 
Thus, by repeating the similar derivation in Appendices~\ref{subsec:FP_equations} and \ref{app_parametrix}, we get the similar upper bound of $\mathbb{E}_{S,S'}\mathrm{KL}(\rho_{\eta}|\gamma_{\eta})$ even when the initial distribution is the Gaussian mixture distribution.

\subsection{Proof of Theorem~\ref{main_thm2_continuous}}
\label{app_convex_bounded}
We first show the proof of Theorem~\ref{main_thm2_continuous}.
\maincontinuous*
\begin{proof}
Since $f(w,z)$ is $R$-strongly convex function for any $z$, we have
\begin{align}
\frac{\partial \mathrm{KL}(\rho_t|\gamma_t)}{\partial t}&\leq -\frac{1}{2\beta}\mathbb{E}\|\nabla \log\rho_t-\nabla\log \gamma_t\|^2 +\frac{\beta}{2} \mathbb{E}\| \nabla F(W_t,S)- \nabla F(W_t,S')\|^{2}\notag \\
&\leq
-\frac{R}{4}\ \mathrm{KL}(\rho_t|\gamma_t) +\frac{\beta}{2} \mathbb{E}\| \nabla F(W_t,S)- \nabla F(W_t,S')\|^{2},
\end{align}
where we utilized the local LSI in Theorem~5.5.2 of \citet{bakry2013analysis}. Since the stationary distribution si $\pi \propto \exp(-\beta F(x))$, From Theorem~5.5.2 of \citet{bakry2013analysis}, $\gamma_t$ satisfies the LSI with the LSI constant $2/(\beta R)$. By integrating $e^{\frac{tR}{4}} \frac{\partial \mathrm{KL}(\rho_{t}|\gamma_{t})}{\partial t}$ over $t \in [0, T]$ and rearranging the above, we obtain the upper bound of $I(W_{t};S)$. This concludes the proof.
\end{proof}

We can obtain the similar result for bounded non-convex losses with $l_2$-regularization $F(w,z) = F_0(w,z) + \frac{\lambda}{2}\|w\|^2$ ($0 < \lambda < \infty$), where $F_0(w,z)$ is $C$-bounded ($0 \leq C < \infty$) with the initial distribution $\pi_0\propto e^{-\frac{\beta\lambda}{2}\|w\|^2}$.
From Lemma~34 in \citet{Li2020On}, $\gamma_t$ satisfies the LSI with the constant $\frac{\lambda}{e^{8\beta C}}$.
Then, following the way in the proof of Theorem~\ref{main_thm2_continuous}, we obtain the bound with replacing $R$ of Eq.~\eqref{eq_stability_generalizaton_con} to $\frac{\lambda}{e^{8\beta C}}$.

\section{Proofs of generalization analyses directly using a training loss}
\label{App_sec_sub_expo_full}
In this section, we provide our proof for our generalization bounds in the case when the same loss is used for training and the generalization performance evaluation (Corollaries~\ref{cor:gen_err_nonsurrogate} and \ref{cor:excess_risk}). 
The key to deriving these bounds is showing that $f$ in SGLD is sub-exponential under Assumptions~\ref{asm_smooth}, \ref{asm_dissipative} (Theorem~\ref{thm:sub_exp}).
Therefore, we explain how to obtain this result in Appendices~\ref{App_sec_sub_expo} and \ref{App_sec_sub_expo_main} before introducing the details of proofs for our bounds in Appendices~\ref{app_proof_gen_without_surrogate} and \ref{subsec:proof_excess}.

\subsection{Preparation for the proof of sub-exponential property}\label{App_sec_sub_expo}
We introduce some auxiliary lemmas that assure the existence of bounded local minima.
These are used later for showing the sub-exponential property of a loss function in SGLD.
\begin{lem}\label{lem_minima}
Suppose that Assumptions~\ref{asm_smooth} and \ref{asm_dissipative} are satisfied.
Then, for each $z\in \mathcal{Z}$, there exists a positive constant $A$ such that
\begin{align*}
|f(0,z)|\leq A.
\end{align*}
\begin{proof}
Denote $\tilde{w}^*_z$ as a global minima of $f(\cdot,z)$ for each $z\in \mathcal{Z}$.
By using Taylor's theorem around $\tilde{w}^*_z$, for $t\in (0,1]$, we obtain the following equation with a parameter $\tilde{w}_z=t \tilde{w}^*_z$:
\begin{align*}
f(0,z)=f(\tilde{w}^*_z,z)+\nabla f(\tilde{w}^*_z,z)\cdot \tilde{w}^*_z+\frac{1}{2}\tilde{w}^*_z\cdot \nabla^2 f(\tilde{w}_z,z)\cdot \tilde{w}^*_z.
\end{align*}
According to Assumption~\ref{asm_smooth} and the fact that $\tilde{w}^*_z$ is the global minima (i.e., $\nabla f(\tilde{w}^*_z,z)=0$), we obtain
\begin{align}\label{eq_lem_assum}
f(0,z)\leq f(\tilde{w}^*_z,z)+\frac{1}{2}M \|\tilde{w}^*_z\|^2.
\end{align}

\citet{Farghly2021} has shown that all the local minima $\tilde{w}^*_z$ are inside the ball in the Euclidean space.
That is, for each $z\in \mathcal{Z}$, all $\tilde{w}^*_z$
are located in $\overline{B(0,r)}$ with $r=\sqrt{b/m}$, where $B(x,r)$ ($r > 0$) is the ball in the Euclidean space defined as
$B(x,r)\coloneqq\{x\in \mathbb{R}^d:\|x-y\|< r\}$
and $\overline{B(x,r)}$ is the closure of $B(x,r)$.
From this fact, we obtain $\|\tilde{w}^*_z\|^2 \leq b/m$ and Eq.~\eqref{eq_lem_assum} can be upper bounded as
\begin{align*}
f(0,z)\leq f(\tilde{w}^*_z,z)+\frac{Mb}{2m}.
\end{align*}

Next, we show that, for each $z\in \mathcal{Z}$, the global minima $f(\tilde{w}^*_z,z)$ is bounded uniformly.
Since $f(\tilde{w}_z,z)$ is continuous with respect to $w$ for each $z\in\mathcal{Z}$ under Assumption~\ref{asm_smooth}, it is continuous in $C$ when considering the closed set $C\coloneqq\overline{B(0,r)}$ with $r=\sqrt{b/m}$.
From the property of the continuous function in the closed set, the maximum and minimum value of $f(\tilde{w}_z,z)$ is always bounded, i.e., we have $f(\tilde{w}^*_z,z)<\infty$ for each $z\in \mathcal{Z}$.
By considering the largest global minimum and denote it as $\tilde{A}$, we obtain $f(\tilde{w}^*_z,z) \leq \tilde{A}$ and thus
\begin{align*}
f(0,z)\leq\tilde{A}+\frac{Mb}{2m}.
\end{align*}
This concludes the proof.
\end{proof}
\end{lem}

Under Lemma~\ref{lem_minima}, we can modify the upper and lower bound for $f(w,z)$ in \citet{Raginsky2017} as follows.
\begin{lem}[Modified version from \citet{Raginsky2017}]\label{lem_function_bound}
Suppose that Assumptions~\ref{asm_smooth} and \ref{asm_dissipative} are satisfied.
Then, for any $z\in\mathcal{Z}$, we have
\begin{align*}
\frac{m}{3}\|w\|^2-\frac{b}{2}\log 3 \leq f(w,z)\leq \frac{M}{2}\|w\|^2+M\sqrt{\frac{b}{m}}\|w\|+A.
\end{align*}
\begin{proof}
    \citet{Raginsky2017} assumed that for any $z\in \mathcal{Z}$, there exists constant $A$ such that $\| f(0,z)\|\leq \ A$. Instead, we show the existence of such $A$ by Lemma~\ref{lem_minima}. On the basis of this fact, we obtain the claim in the same way with \citet{Raginsky2017}.
\end{proof}
\end{lem}

\subsection{Sub-exponential property for a loss function}\label{App_sec_sub_expo_main}
We now provide the complete proof of Theorem~\ref{thm:sub_exp}.
We first show the fact that a loss function in SGLD has the sub-exponential property (Appendix~\ref{App_sec_sub_expo_main}) and explain how to evaluate the constants in the sub-exponential condition for deriving our generalization bounds (Appendix~\ref{app_proof_gen_without_surrogate}).

Recall that the statement of Theorem~\ref{thm:sub_exp} is as follows.
\subexp*
Now, we proceed to the proof of sub-exponential property.
To show the sub-exponential property for $f(w,z)$, it is sufficient to show that there exists a positive number $c_0$ such that $\mathbb{E}e^{\lambda f(w,z)}<\infty$ for all $|\lambda|\leq c_0$ (see Theorem~2.13 in \citet{wainwright_2019}). To show this, we follow the proof of Proposition~2.7.1 in \citet{vershynin2018high}, which uses the Taylor expansion of the exponential moment. By considering the Taylor expansion, we have
\begin{align}
\mathbb{E} e^{\lambda(f(W_T,Z)-\mathbb{E}f(W_T,Z))}&\leq 1+\mathbb{E}\sum_{p=2}^\infty \frac{\lambda^p (f(W_T,Z)-\mathbb{E}f(W_T,Z))^p}{p!}\notag \\
&\leq 1+\mathbb{E}\sum_{p=2}^\infty \frac{(\lambda e)^p (f(W_T,Z)-\mathbb{E}f(W_T,Z))^p}{p^p}
\end{align}
where we used $p!\geq (p/e)^p$, which is obtained by the Stirling's approximation. Later, we restrict the $\lambda$ such that this series converges, and thus, we can swap the sum and expectation. 
From the fact that $(x+y)^p \leq 2^{p-1}x^p+2^{p-1}y^p$ for $x,y \geq 0$, we obtain
\begin{align}
\mathbb{E}(f(W_T,Z)-\mathbb{E}f(W_T,Z))^p\leq 2^{p-1}\mathbb{E}[|f(W_T,Z)|^p]+2^{p-1}(|-\mathbb{E}[f(W_T,Z)]|)^p.
\end{align}
Given the marginal distribution of the parameters obtained by the $T$-th iterate of the SGLD algorithm, i.e., $W_{T} \sim p(W_T)$, we have the following fact by using the result of Lemma~\ref{lem_function_bound} and the Cauchy--Schwartz inequality:
\begin{align}\label{eq:_f_bound}
\frac{m}{3}\|W_{T}\|^2-\frac{b}{2}\log 3 \leq f(W_{T},Z) \leq \frac{M}{2}\|W_{T}\|^2+M\sqrt{\frac{b}{m}}\|W_{T}\|+A \leq  M\|W_{T}\|^2+\frac{b}{2m}+A.
\end{align}

By using Eq.~\eqref{eq:_f_bound} and the inequality $(x+y)^p \leq 2^{p-1}x^p+2^{p-1}y^p$ for $x,y \geq 0$, we have
\begin{align}
&\mathbb{E} f(W_{T},Z)^p\leq \mathbb{E} |f(W_{T},Z)|^p \notag \\
&\leq  \left(2^{p-1}M^p\mathbb{E} \|W_{T}\|_{2}^{2p}+2^{p-1}\bigg(\frac{b}{2m}+A\bigg)^p\right)\vee\left(2^{p-1}\bigg(\frac{m}{3}\bigg)^p\mathbb{E} \|W_{T}\|_{2}^{2p}+2^{p-1}\bigg(\frac{b}{2}\log 3\bigg)^p\right).
\end{align}

We then use the following lemma, which is adapted from Lemma~8 in \citet{Mou22}:
\begin{lem}\label{lmm_app_pth_moment}
Suppose that Assumptions ~\ref{asm_smooth}, ~\ref{asm_dissipative} and \ref{regularity_FP} are satisfied.
Then, for all $T\in\mathbb{N}$ and all $p\in\mathbb{N}$, there is a universal constant $C>0$ that satisfies
\begin{align}
(\mathbb{E}\|W_T\|_2^{p})^{1/p}\leq C\left(\mathbb{E}\|W_0\|_2^{p}\right) ^{\frac{1}{p}}+C\sqrt{\frac{p+\beta b+d}{\beta m}}.
\end{align}
\end{lem}
Note that the original lemma in \citet{Mou22} is shown for the Langevin diffusion with no stochastic gradient descent; however, the bound in \citet{Mou22} also holds in the SGLD setting because we assumed the dissipativity for \emph{each data point} in Assumption~\ref{asm_dissipative}. 

From Proposition~2.5.2 in \cite{vershynin2018high}, if for any $p\geq 1$, $L_p$ norm of a random variable $X$ is bounded as $(\mathbb{E}[X^p])^{1/p}\leq C \sqrt{p}$ with some positive constant $C$, then $X$ is sub-Gaussian random variable. From Lemma~\ref{lmm_app_pth_moment}, it is clear that $W_T$ is a sub-Gaussian random variable. Note that constant terms such as $C\left(\mathbb{E}\|W_0\|_2^{p}\right) ^{\frac{1}{p}}$ can be upper bounded by $\sqrt{p}$ multiplied by some positive constants.

Thus, we have
\begin{align}
\mathbb{E}\|W_T\|_2^{2p}&\leq 2^{2p-1}C^{2p}\left(\mathbb{E}\|W_0\|_2\right)^{2p} +2^{2p-1}C^{2p}\left(\frac{p+\beta^2b+d}{\beta^2m}\right)^p\notag \\
&\leq 2^{2p-1}C^{2p}\left(\mathbb{E}\|W_0\|_2\right)^{2p} +2^{3p-2}C^{2p}\left(\frac{\beta^2b+d}{\beta^2m}\right)^p+2^{3p-2}C^{2p}(\beta^2m)^{-p}p^p,
\end{align}
and thus
\begin{align}
\mathbb{E} f(W_{T},Z)^p\leq C_0^p+C_1^pp^p,
\end{align}
where $C_0$ and $C_1$ are positive constants that only depend on $s^2, m, M, b, d, A$, and $\beta$.
For the latter purpose, we introduce $C_5$ as
\begin{align}
\mathbb{E} f(W_{T},Z)^p\leq C_5p^p,
\end{align}
where $C_5$ only depends on $s^2, m, M, b, d, A$ and $\beta$.
Then, we have
\begin{align}
\mathbb{E} e^{\lambda(f(W_T,Z)-\mathbb{E}f(W_T,Z))}&\leq 1+\sum_{p=2}^\infty \frac{(\lambda e)^p C_5p^p}{p^p}= 1+\sum_{p=2}^\infty (\lambda eC_5)^p=1+\frac{(\lambda eC_5)^2}{1-\lambda eC_5}, 
\end{align}
where $\lambda eC_5<1$. Moreover, by setting $\lambda eC_5<1/2$, we have
\begin{align}
\mathbb{E} e^{\lambda(f(W_T,Z)-\mathbb{E}f(W_T,Z))}&\leq 1+2\lambda^2 e^2C_5^2\leq e^{2\lambda^2e^2C_5^2}.
\end{align}
From the above, we can see that $f(W_{T},Z)$ is a sub-exponential function with the following constants: $\sigma_e^2\coloneqq4e^2C_5^2$ and $\nu\coloneqq\frac{1}{2eC_5}$ where $C_5$ only depends on $s^2, m, M, b, d$ and $A$.

\subsection{Proof of generalization error bound directly using a training loss}
\label{app_proof_gen_without_surrogate}
Here, we provide the complete proof of Corollary~\ref{cor:gen_err_nonsurrogate}.
\nonsurrogate*
\begin{proof}
We use the following theorem in \citet{bu2020tightening} to derive the generalization error for sub-exponential losses.
\begin{thm}[\citet{bu2020tightening}]
\label{thm:bu}
  Suppose that there exist positive constants $\sigma^2_e$ and $\nu$ such that
  \begin{align*}
      \log \mathbb{E}_{W_T\otimes Z}\left[e^{\lambda (f(W_T,Z)-\mathbb{E}_{W_T\otimes Z}[f(W_T,Z)])}\right]\leq  \frac{\sigma^2_e\lambda^2}{2}\quad for\ all\ |\lambda|<\frac{1}{\nu}.
  \end{align*}
Then, we have
\begin{align}
|\mathrm{gen}(\mu,P_{W_T|S};F)| \leq \Psi^{*-1}\left(\frac{I(W_T;S)}{n}\right),
\end{align}
where
\begin{align}
    \Psi^{*-1}(y)=\begin{cases}
    \sqrt{2\sigma_e^2 y}\quad \mathrm{if}\ y\leq \frac{\sigma_e^2}{2\nu}\\
    \nu y+\frac{\sigma_e^2}{2\nu}\quad \mathrm{otherwise}.
    \end{cases}
\end{align}  
\end{thm}
Substituting the constants of the sub-exponential property shown in Theorem~\ref{thm:sub_exp} and the upper bound of $I(W_T;S)$ in Theorem~\ref{main_thm1_time_independent} into the above completes the proof.
\end{proof}

\subsection{Proof of an excess risk}
\label{subsec:proof_excess}
We rewrite our corollary as follows.
\excess*
\begin{proof}
We can decompose the excess risk at $T$ as
\begin{align}
\mathrm{Excess}(\mu,P_{W_{T}|S})&=\mathbb{E}_{W_{T},S}[F_\mu(W_{T})-F_S(W_{T})+F_S(W_{T})-F_\mu(w^*)] \\
&=\mathrm{gen}(\mu,P_{W_{T}|S};F)+ \mathbb{E}_{W_{T},S}[F_S(W_{T})-F_\mu(w^*)],
\end{align}
where the last term is called the \emph{optimization error}.
The optimization error can be bounded as
\begin{align}
\label{eq:opt_err_bound}
\mathbb{E}_{W_{T},S}[F_S(W_{T})-F_\mu(w^*)]&=\mathbb{E}_{W_{T},S}[F_S(W_{T})-\min_{w}F_S(w)+\min_{w}F_S(w)-F_S(w^*)] \\
&\leq \mathbb{E}_{W_{T},S}[F_S(W_{T})-\min_{w}F_S(w)],
\end{align}
where the above inequality comes from the fact that $\mathbb{E}_{W_{T},S}[\min_{w}F_S(w) -F_S(w^*)]\leq 0$.
Let us denote $\epsilon_{\mathrm{opt}}$ as $\mathbb{E}_{W_{T},S}[F_S(W_{T})-\min_{w}F_S(w)]$.
Then, we can express the upper bound of the excess risk as follows:
\begin{align*}
\mathrm{Excess}(\mu,P_{W_{T}|S}) \leq |\mathrm{gen}(\mu,P_{W_{T}|S};F)|+ \epsilon_{\mathrm{opt}}.
\end{align*}

We first bound the $\epsilon_{\mathrm{opt}}$ term.
Using the Gibbs distribution $\pi(\mathrm{d}w) \propto \exp(-\beta F(w,S))$ and the triangle inequality, we obtain
\begin{align}
\label{eq:opt_bound}
\epsilon_{\mathrm{opt}}\!\leq\! |\mathbb{E}_{W_{T},S}F_S(W_{T})\!-\!\mathbb{E}_{\pi,S}F_S(W)|\!+\!|\mathbb{E}_{\pi,S}F_S(W)\!-\!\min_{w}F_S(w)|,
\end{align}
where we express the expectation under the joint distribution $\mu^N \otimes \pi$ as $\mathbb{E}_{\pi,S}$.
The first term on the right-hand side of Eq.~\eqref{eq:opt_bound} is the \emph{convergence error} of the SGLD algorithm, which can be seen as $\mathcal{O}(e^{-k\eta/\beta c_{LS}}+\sqrt{\eta})$. 
From Lemma~6 in \citet{Raginsky2017}, we have
\begin{align}
 |\mathbb{E}_{W_{T},S}F_S(W_{T})\!-\!\mathbb{E}_{\pi,S}F_S(W_{T})| \leq \left(M\sigma+M\sqrt{\frac{b}{m}}\right)\mathbb{E}_{S}W_2(P_{W_{T}|S},\pi),
\end{align}
where $\sigma^2\coloneqq \mathbb{E}_{P_{W_{T}|S}}\|W_T\|^2\vee \mathbb{E}_{\pi}\|W\|^2$ and $W_2$ is the $2$-Wasserstein distance. 
By using the $T_2$ inequality, we obtain $W_2(P_{W_{T}|S},\pi) \leq \sqrt{c_{\mathrm{LS}}\mathrm{KL}(P_{W_{T}|S}\|\pi)}$. 
From Theorem~1 in \citet{vempala2019rapid}, we further obtain $\mathrm{KL}(P_{W_{T}|S}\|\pi)\leq \mathcal{O}(e^{-2k\eta/\beta c_{LS}}+\eta)$. 
Combining these results leads to $|\mathbb{E}_{W_{T},S}F_S(W_{T})\!-\!\mathbb{E}_{\pi,S}F_S(W_{T})| = \mathcal{O}(e^{-k\eta/\beta c_{LS}}+\sqrt{\eta})$.

The second term $|\mathbb{E}_{\pi,S}F_S(W_{T})\!-\!\min_{w}F_S(w)|$ corresponds to the \emph{minimization error}, which can be upper-bounded by $c_{\mathrm{err}}\coloneqq\frac{d}{2\beta}\log\left(\frac{e M}{m}\left(\frac{b\beta}{d}+1\right)\right)$ according to Proposition~11 in \citet{Raginsky2017}.
This completes the proof.
\end{proof}

\end{document}